\documentclass{article}
\usepackage[authoryear,sort]{natbib}
\usepackage[bitstream-charter]{mathdesign}
\usepackage[
  paper  = letterpaper,
  left   = 1.2in,
  right  = 1.2in,
  top    = 1.0in,
  bottom = 1.0in,
  ]{geometry}

\usepackage{amsthm,amsmath}

\usepackage[english]{babel}
\usepackage[parfill]{parskip}
\usepackage{afterpage}
\usepackage{framed}

\usepackage{graphicx}
\usepackage[labelfont=bf]{caption}
\usepackage{booktabs}

\usepackage{listings}
\usepackage{fancyvrb}
\fvset{fontsize=\normalsize}

\usepackage[colorlinks,
            linkcolor = blue,
            urlcolor  = blue,
            citecolor = blue,
            linktoc=all]{hyperref}

\usepackage[nameinlink]{cleveref}

\usepackage[acronym,nowarn,nohypertypes={acronym,notation}]{glossaries}
\setacronymstyle{long-sc-short}

\lstdefinestyle{mystyle}{
    commentstyle=\color{OliveGreen},
    keywordstyle=\color{BurntOrange},
    numberstyle=\tiny\color{black!60},
    stringstyle=\color{MidnightBlue},
    basicstyle=\ttfamily,
    breakatwhitespace=false,
    breaklines=true,
    captionpos=b,
    keepspaces=true,
    numbers=left,
    numbersep=5pt,
    showspaces=false,
    showstringspaces=false,
    showtabs=false,
    tabsize=2
}
\usepackage{tikz}
\usetikzlibrary{shapes,decorations,arrows,calc,arrows.meta,fit,positioning}
\tikzset{
    -Latex,auto,node distance =1 cm and 1 cm,semithick,
    state/.style ={circle, draw, minimum width = 0.7 cm},
    detstate/.style ={rectangle, draw, minimum width = 0.7 cm, minimum height = 0.7 cm},
    point/.style = {circle, draw, inner sep=0.04cm,fill,node contents={}},
    bidirected/.style={Latex-Latex,dashed},
    el/.style = {inner sep=2pt, align=left, sloped}
}

\usepackage{wrapfig}
\usepackage{mathtools}

\usepackage{mfirstuc}
\usepackage{xspace}

\makeatletter
\def\adl@drawiv#1#2#3{%
        \hskip.5\tabcolsep
        \xleaders#3{#2.5\@tempdimb #1{1}#2.5\@tempdimb}%
                #2\z@ plus1fil minus1fil\relax
        \hskip.5\tabcolsep}
\newcommand{\cdashlinelr}[1]{%
  \noalign{\vskip\aboverulesep
           \global\let\@dashdrawstore\adl@draw
           \global\let\adl@draw\adl@drawiv}
  \cdashline{#1}
  \noalign{\global\let\adl@draw\@dashdrawstore
           \vskip\belowrulesep}}
\makeatother

\DeclareRobustCommand{\mb}[1]{\ensuremath{\boldsymbol{\mathbf{#1}}}}

\DeclareMathOperator*{\argmin}{arg\,min}

\renewcommand{\mid}{~\vert~}

\newcommand{\mbu}{\mb{u}}
\newcommand{\mbv}{\mb{v}}
\newcommand{\mbw}{\mb{w}}
\newcommand{\mbx}{\mb{x}}
\newcommand{\mby}{\mb{y}}
\newcommand{\mbz}{\mb{z}}

\newcommand{\mbI}{\mb{I}}

\newcommand{\mbS}{\mb{S}}

\newcommand{\mbU}{\mb{U}}
\newcommand{\mbV}{\mb{V}}

\newcommand{\mbX}{\mb{X}}
\newcommand{\mbY}{\mb{Y}}

\newcommand{\mbdelta}{\mb{\delta}}

\newcommand{\cF}{\mathcal{F}}

\newcommand{\cN}{\mathcal{N}}

\newcommand{\E}{\mathbb{E}}

\newcommand{\g}{\mid}

\newtheorem{thmdef}{Definition}
\newtheorem{lemma}{Lemma}
\newtheorem{corr}{Corollary}
\newtheorem{thm}{Theorem}

\newtheorem*{lemma*}{Lemma}

\crefname{lemma}{lemma}{lemmas}
\crefname{prop}{proposition}{propositions}

\newcommand{\indep}{\rotatebox[origin=c]{90}{$\models$}}

\newcommand{\ind}{\mathbf{1}}

\newcommand{\ptr}{{p_{tr}}}

\newcommand{\pte}{{p_{te}}}

\newcommand{\ftheta}{{f_\theta}}

\newacronym{KL}{kl}{Kullback-Leibler}
\newacronym{ELBO}{elbo}{\emph{evidence lower bound}}
\newacronym{POPELBO}{pop-elbo}{\emph{population evidence lower bound}}

\newacronym{SVI}{svi}{stochastic variational inference}
\newacronym{BUMPVI}{bump-vi}{bumping variational inference}

\newacronym{GMM}{gmm}{Gaussian mixture model}
\newacronym{LDA}{lda}{latent Dirichlet allocation}

\newacronym{SUTVA}{sutva}{stable unit treatment value assumption}

\newacronym{KSD}{ksd}{{kernelized Stein discrepancy}}
\newacronym{KCC-SD}{kcc-sd}{kernelized complete conditional Stein discrepancy}

\newacronym{OPVI}{opvi}{operator variational inference}
\newacronym{SVGD}{svgd}{Stein variational gradient descent}

\newacronym{erm}{erm}{Empirical Risk minimization}

\newacronym{nurd}{NuRD}{Nuisance-Randomized Distillation}
\newacronym{jtt}{jtt}{Just Train Twice}
\newacronym{lff}{lff}{Learning from Failure}
\newacronym{nli}{nli}{natural language inference}

\newacronym{lr}{lr}{learning rate}
\newacronym{wd}{wd}{weight decay}
\newacronym{pr}{pr}{patch randomization}
\newacronym{nr}{nr}{n-gram randomization}

\newacronym{ood}{ood}{out-of distribution}
\newacronym{sd}{sd}{spectral decoupling}
\newacronym{ntk}{ntk}{neural tangent kernel}

\newacronym{cnc}{cnc}{Correct-n-Contrast}
\newacronym{dfr}{dfr}{Deep Feature Reweighting}
\newcommand{\sms}{shortcut-mitigating methods}

\newcommand{\sigdamp}{$\sigma$-damp}
\newcommand{\stitch}{$\sigma$-stitch}

\newacronym{vdm}{marg-mech}{margin trading mechanism}
\newacronym{dgp}{dgp}{data generating process}
\newacronym{cc}{marg-ctrl}{margin control}
\newacronym{sm}{sm}{shortcut-mitigating method}
\newacronym{gd}{gd}{gradient descent}
\newacronym{grad-opt}{grad-opt}{gradient-based optimization}

\newcommand{\derm}{default-\textsc{erm}}
\newcommand{\Derm}{Default-\textsc{erm}}

\newacronym{mech}{marg-mech}{the margin trading mechanism}

\usepackage{nicefrac}
\usepackage{subcaption}
\usepackage{wrapfig}
\usepackage{titlesec}

\newcommand{\keypoint}[1]{\textbf{#1}}

\usepackage[utf8]{inputenc} 
\usepackage[T1]{fontenc}    
\usepackage{hyperref}       
\usepackage{url}            
\usepackage{booktabs}       
\usepackage{nicefrac}       
\usepackage{xcolor}        
\usepackage{soul}
\usepackage{float}

\usepackage{lineno}

\usepackage[framemethod=tikz]{mdframed}

\usepackage{titling}

\title{
\textbf{Don’t blame Dataset Shift! \\ Shortcut Learning due to Gradients and Cross Entropy}
  \vspace{20pt}
}

\author{
 Aahlad Puli\textsuperscript{1}\thanks{\textsuperscript{1}Corresponding email: \texttt{aahlad@nyu.edu}.}
  \hspace{15pt} Lily Zhang \textsuperscript{2} \hspace{15pt} Yoav Wald \textsuperscript{2}\hspace{15pt} 
  Rajesh Ranganath\textsuperscript{1,2,3}
     \\\\
     \hspace{-10pt} 
     \textsuperscript{1}Department of Computer Science, New York University  \\
     \hspace{-10pt} 
     \textsuperscript{2}Center for Data Science, New York University \\
     \hspace{-10pt} 
     \textsuperscript{3}Department of Population Health, Langone Health, New York University
}

\date{}

\begin{document}

\maketitle

\begin{abstract}
Common explanations for shortcut learning assume that the shortcut improves prediction under the training distribution but not in the test distribution.
Thus, models trained via the typical gradient-based optimization of cross-entropy, which we call \derm{}, utilize the shortcut.
However, even when the stable feature determines the label in the training distribution and the shortcut does not provide any additional information, like in perception tasks, \derm{} still exhibits shortcut learning.
Why are such solutions preferred when the loss for \derm{} can be driven to zero using the stable feature alone? 
By studying a linear perception task, we show that \derm{}'s preference for maximizing the margin 
leads to models that depend more on the shortcut than the stable feature, even without overparameterization.
This insight suggests that \derm{}’s implicit inductive bias towards max-margin is unsuitable for perception tasks.
Instead, we develop an inductive bias toward uniform margins and show that this bias guarantees dependence only on the perfect stable feature in the linear perception task.
We develop loss functions that encourage uniform-margin solutions, called \gls{cc}.
\Gls{cc} mitigates shortcut learning on a variety of vision and language tasks, showing that better inductive biases can remove the need for expensive two-stage \sms{} in perception tasks.
\end{abstract}

\section{Introduction}\label{sec:intro}
Shortcut learning is a phenomenon where a model learns to base its predictions on an unstable correlation, or \textit{shortcut}, that does not hold across data distributions collected at different times and/or places \citep{geirhos2020shortcut}.
A model that learns shortcuts can perform worse than random guessing in settings where the label's relationship with the shortcut feature changes \citep{pmlr-v139-koh21a,puli2022outofdistribution}.
Such drops in performance do not occur if the model depends on features whose relationship with the label does not change across settings; these are \emph{stable} features.

Shortcut learning is well studied in cases where models that use both shortcut and stable features achieve lower loss than models that only use the stable feature \citep{arjovsky2019invariant,puli2022outofdistribution,geirhos2020shortcut}.
These works consider cases where the Bayes-optimal classifier --- the training conditional distribution of the label given the covariates --- depends on both stable and shortcut features.
In such cases, shortcut learning  occurs as the Bayes-optimal predictor is the target of standard supervised learning algorithms such as the one that minimizes the log-loss via \gls{gd}, which we call \derm{}.

However, in many machine learning tasks, the stable feature perfectly predicts the label, i.e. a \textit{perfect stable feature}.
For example, in task of predicting hair color from images of celebrity faces in the CelebA dataset \citep{sagawa2019distributionally}, the color of the hair in the image determines the label. This task is a perception task.
In such classification tasks, the label is independent of the shortcut feature given the stable feature, and  the Bayes-optimal predictor under the training distribution only depends on the stable feature.
\Derm{} can learn this Bayes-optimal classifier which, by depending solely on the stable feature, also generalizes outside the training distribution.
But in practice, \derm{} run on finite data yields models that depend on the shortcut and thus perform worse than chance outside the training distribution~\citep{sagawa2019distributionally,liu2021just,pmlr-v162-zhang22z}.
The question is, why does \derm{} prefer models that exploit the shortcut even when a model can achieve zero loss using the stable feature alone?

To understand preferences toward shortcuts, we study \derm{} on a linear perception task with a stable feature that determines the label and a shortcut feature that does not.
The perfect linear stable feature means that data is linearly separable. This separability means that \derm{-trained} linear models classify in the same way as the minimum $\ell_2$-norm solution that has all margins greater than $1$; the latter is commonly called max-margin classification \citep{soudry2018implicit}.
We prove that \derm{}'s implicit inductive bias toward the max-margin solution is harmful in that \derm{-trained} linear models depend more on the shortcut than the stable feature.
{In fact, such dependence on the shortcut occurs even in the setting with fewer parameters in the linear model than data points, i.e. without overparameterization.}
These observations suggest that a max-margin inductive bias is unsuitable for perception tasks.

Next, we study inductive biases more suitable for perception tasks with perfect stable features.
We first observe that predicting with the perfect stable feature alone achieves uniform margins on all samples.
Formally, if the stable feature $s(\mbx)$ determines the label $\mby$ via a function $d$, $\mby= d\circ s(\mbx)$, one can achieve any positive $b$ as the margin on all samples simultaneously by predicting with $b \cdot d\circ s(\mbx)$.
We show that in the same setting
without overparameterization
where max-margin classification leads to shortcut learning,
models that classify with uniform margins depend only on the stable feature.

Building on these observations, we identify alternative loss functions that are inductively biased toward uniform margins, which we call \glsreset{cc}\gls{cc}.
We empirically demonstrate that
\gls{cc} mitigates shortcut learning on multiple vision and language tasks without the use of annotations of the shortcut feature in training.
Further, \Gls{cc} performs on par or better than the more expensive two-stage \sms{} \citep{liu2021just,pmlr-v162-zhang22z}.
We then introduce a more challenging setting where both training and validation shortcut annotations are unavailable, called the nuisance-free setting.
In the nuisance-free setting, \gls{cc} \emph{always outperforms} \derm{} and the two-stage \sms{}.
These empirical results suggest that simply incorporating inductive biases more suitable for perception tasks is sufficient to mitigate shortcuts.

\section{Shortcut learning in perception tasks due to maximizing margins}
\label{sec:method}

\paragraph{Setup.} We use $\mby, \mbz, \mbx$ to denote the label, the shortcut feature, and the covariates respectively.
We let the training and test distributions ($\ptr, \pte$) be members of a family of distributions indexed by $\rho$, $\cF=\{p_\rho(\mby, \mbz, \mbx)\}_{\rho}$, such that the shortcut-label relationship $p_\rho(\mbz, \mby)$ changes over the family.
Many common tasks in the spurious correlations literature have stable features $s(\mbx)$ that are perfect, meaning that the label is a deterministic function $d$ of the stable feature: $\mby=d\circ s(\mbx)$.
For example, in the Waterbirds task the bird's body determines the label and in the CelebA task, hair color determines the label \citep{sagawa2019distributionally}.
As $s(\mbx)$ determines the label, it holds that $\mby \indep_{p_\rho} (\mbx, \mbz) \g s(\mbx)$.
Then, the optimal predictor on the training distribution is optimal on all distributions in the family $\cF$, regardless of the shortcut because
$\,\, \ptr(\mby \g \mbx) = 
\ptr(\mby \g s(\mbx)) = 
\pte(\mby \g s(\mbx))=
\pte(\mby \g \mbx).$ 

The most common procedure to train predictive models to approximate $\ptr( \mby\g \mbx )$ is gradient-based optimization of cross-entropy (also called log-loss); we call this \derm{}.
\Derm{} targets the Bayes-optimal predictor of the training distribution which, in tasks with perfect stable features, also performs optimally under the test distribution.
However, despite targeting the predictor that does not depend on the shortcut, models built with \derm{} still rely on shortcut features that are often less predictive of the label and are unstable, i.e. vary across distributions \citep{geirhos2020shortcut,puli2022outofdistribution}.
We study \derm{'s} preference for shortcuts in a \gls{dgp} where both the shortcut and the perfect stable feature are linear functions of the covariates.

\subsection{Shortcut learning in linear perception tasks}

Let $\textrm{Rad}$ be the uniform distribution over $\{1,-1\}$, $\cN$ be the normal distribution, $d$ be the dimension of $\mbx$, and $\rho\in (0,1),B>1$ be scalar constants.
The \gls{dgp} for $p_\rho(\mby, \mbz, \mbx)$ is:
\begin{align}
\label{eq:sim-example}
    \mby \sim \textrm{Rad}, 
    \quad
    \mbz 
    \sim \begin{cases}
            p_\rho(\mbz = y \g \mby = y) = \rho \\
            p_\rho(\mbz = -y \g \mby = y) = (1 - \rho) \\
            \end{cases},
\quad \mbdelta  \sim \cN(0, \mathbf{I}^{d-2}),
    \quad 
    \mbx 
    = \left[B*\mbz,\mby,\mbdelta \right].
\end{align}

This \gls{dgp} is set up to mirror the empirical evidence in the literature showing that shortcut features are typically learned first \citep{sagawa2019distributionally}.
The first dimension of $\mbx$, i.e. $\mbx_1$, is a shortcut that is correlated with $\mby$ according to $\rho$.
The factor $B$ in $\mbx_1$ scales up the gradients for parameters that interact with $\mbx_1$ in predictions.
For large enough $B$, model dependence on the shortcut feature during \derm{} goes up faster than the stable feature \citep{idrissi2021simple}.

\begin{wrapfigure}[28]{R}{0.55\textwidth}
\small
\vspace{-22pt}
\centering
\hspace{-8pt}
   \begin{subfigure}[b]{0.57\textwidth}
\includegraphics[width=.99\textwidth]{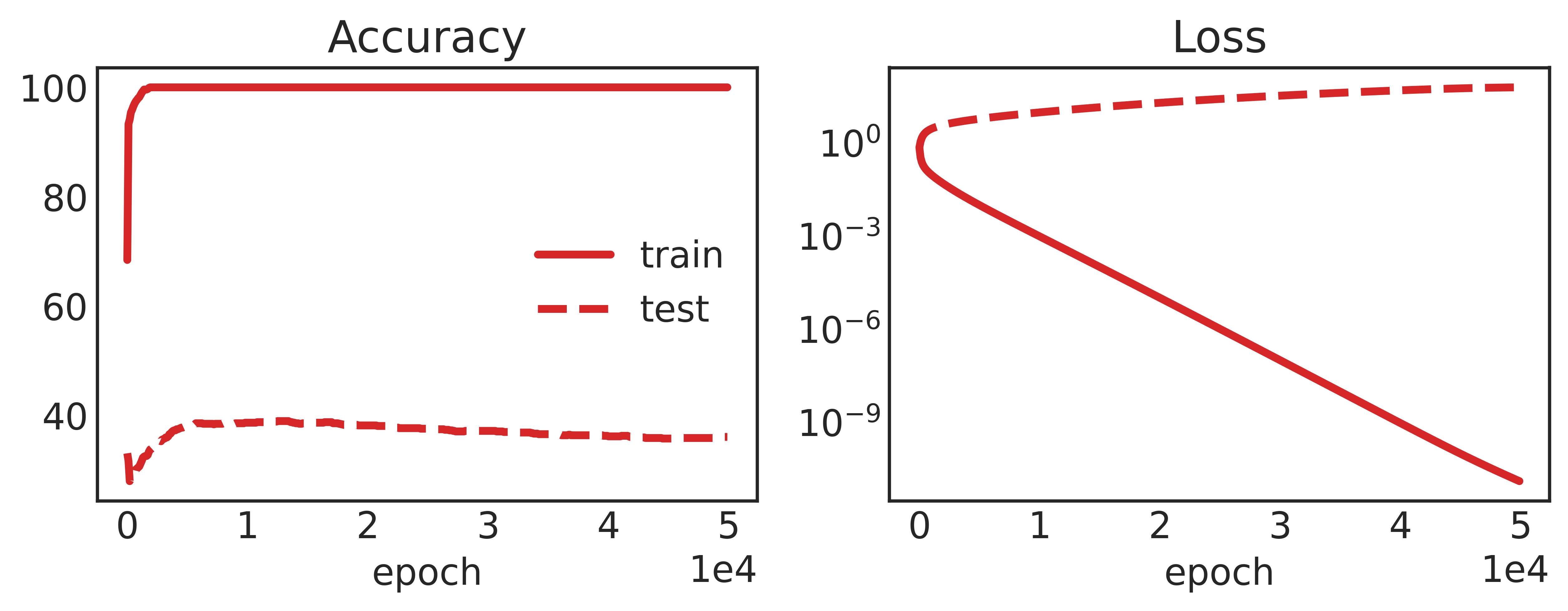}
\vspace{-5pt}
         \caption{
Average accuracy and loss curves.
}
         \label{fig:lin-sim-example}
     \end{subfigure}
\begin{subfigure}[b]{0.57\textwidth}
\vspace{5pt}
\includegraphics[width=\textwidth]{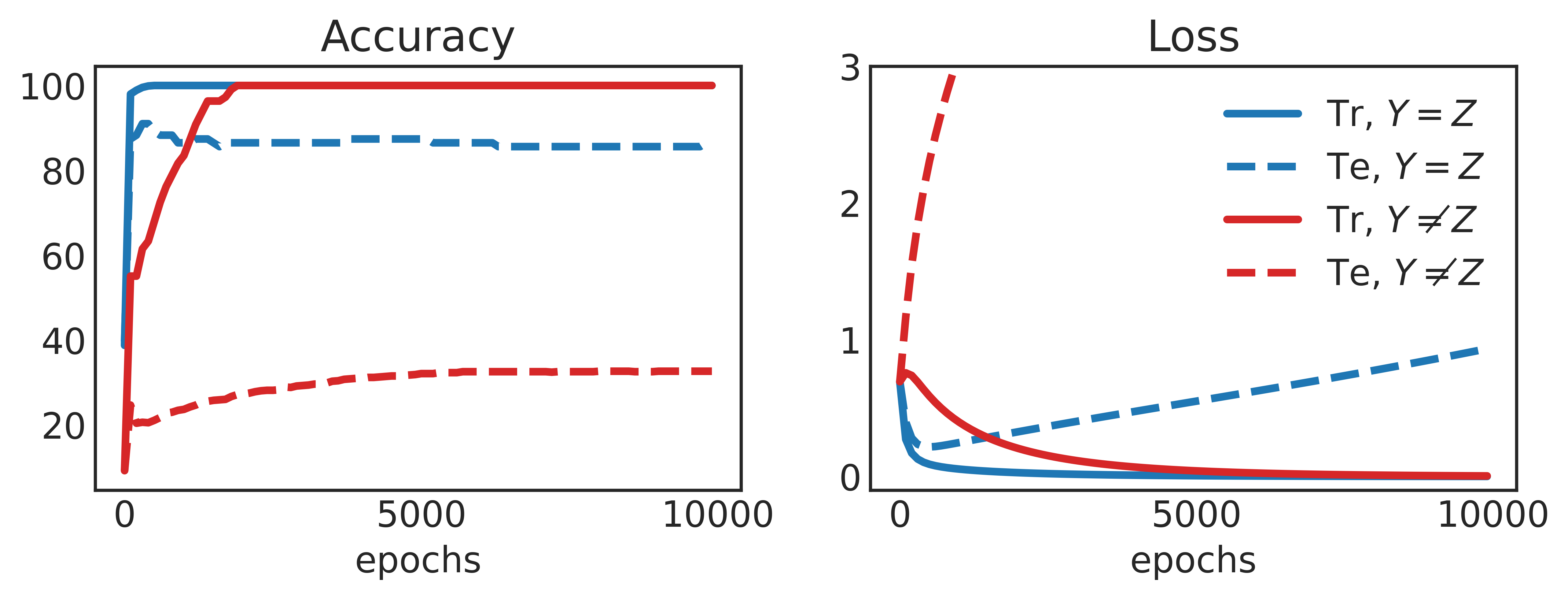}
\vspace{-10pt}
     \caption{Accuracy and loss on shortcut and leftover groups.}
\label{fig:lin-sim-example-pergroup}
 \end{subfigure}
\vspace{-15pt}
\caption{
\small
Accuracy and loss curves for training a linear model with \derm{} on $1000$ training samples from $p_{0.9}$, with $B=10, d=300$ (see \cref{eq:sim-example}), and testing on $p_{0.1}$.
\textbf{(a)} The model achieves $100\%$ train accuracy but $<40\%$ test accuracy. \textbf{(b)} The learned model achieves high test accuracy ($\approx 90\%$) on the shortcut group and low test accuracy on the leftover group ($\approx 30\%$). Models that depend more on the stable feature than on the shortcut, achieve at least $50\%$ accuracy on both the shortcut and leftover groups.
Hence the learned model exploits the shortcut to classify the shortcut group and overfits to the leftover group. 
}
	\label{fig:lin-erm-failure}
 \vspace{-35pt}
\end{wrapfigure}

The training distribution is $\ptr=p_{0.9}$ and the test distribution is one where the shortcut's relationship with the label is flipped $\pte=p_{0.1}$.
Models achieve worse than random test accuracy ($50\%$) if they exploit the training shortcut relationship and the predicted class flips when the shortcut feature flips.
We train with \derm{} which uses log-loss:  on a data point $(\mbx, \mby)$ the log-loss is
\begin{align*}
\ell_{log}(\mby\ftheta(\mbx)) = \log \left[1 + \exp(-\mby \ftheta(\mbx))\right].
\end{align*}
With $d=300$ and $B=10$, we train a linear model on $1000$ samples from the training distribution $p_{\rho=0.9}$, and evaluate on $1000$ samples from $p_{\rho=0.1}$.

\paragraph{Observations.}
\Cref{fig:lin-sim-example} shows that when trained with \derm{}, the linear model does not do better than chance ($<50\%$) on the test data even after $50,000$ epochs.
So, even in the presence of the perfect feature $\mbx_2$, the model relies on other features like the shortcut $\mbx_1$. 
Since the final training loss is very small, on the order of $10^{-9}$, this result is not due to optimization being stuck in a local minima with high loss.
\keypoint{These observations indicate that, in the linear setting, gradient-based optimization with log-loss prefers models that depend more on the shortcut than the perfect stable feature.}

To better understand this preference we focus on the errors in specific groups in the data.
Consider the classifier that only uses the shortcut $\mbz$ and makes the Bayes-optimal prediction w.r.t $\ptr$: $\mathrm{arg}\max_{y}{\ptr(\mby=y \g \mbz)}$.
We call instances that are classified correctly by this model the \textit{shortcut} group, and the rest the \textit{leftover} group.
We use these terms for instances in the training set as well as the test set.
In this experiment $\mby$ is positively correlated with $\mbz$, hence the shortcut group consists of all instances with $\mby^i=\mbz^i$ and the leftover group of those with $\mby^i\neq\mbz^i$.

\Cref{fig:lin-sim-example-pergroup} gives accuracy and loss curves on the shortcut and leftover groups for the first $10000$ epochs.
The test accuracy for the shortcut group hits $90\%$ while the leftover group test accuracy is $<40\%$, meaning that the model exploits the shortcuts.
Even though a model that relies solely on the shortcut misclassifies the leftover group, we see that the training loss of the learned model on this group approaches $0$. 
The model drives down training loss in the leftover group by depending on  noise, which results in larger test loss in the leftover group than the shortcut group.
\keypoint{Thus, \cref{fig:lin-sim-example-pergroup} demonstrates that the \derm{-trained} model classifies the training shortcut group by using the shortcut feature while overfitting to the training leftover group.}

Shortcut dependence like in \cref{fig:lin-erm-failure} occurs even with $\ell_2$-regularization and when training neural networks; see \cref{appsec:l2-reg} and \cref{cc-on-nn} respectively.
Next, we analyze the failure mode in \cref{fig:lin-erm-failure}, showing that the shortcut dependence is due to \derm{}'s implicit bias to learn the max-margin classifier.
Next, we study the failure mode in \cref{fig:lin-erm-failure} theoretically, showing that the shortcut dependence is due to \derm{}'s inductive bias toward learning the max-margin classifier.

\paragraph{Max-margin classifiers depend more on the the shortcut than the stable feature.}

We consider training a linear model $f_\theta(\mbx) = \mbw^\top \mbx$ where $\mbw = [\mbw_z, \mbw_y, \mbw_e]$ with \derm{}.
Data from \cref{eq:sim-example} is always linearly separable due to the perfect stable feature, but many hyperplanes that separate the two classes exist.
When a linear model is trained with \derm{} on linearly separable data, it achieves zero training loss and converges to the direction of a minimum $\ell_2$-norm solution that achieves a margin of at least $1$ on all samples \citep{soudry2018implicit,wang2021implicit,wang2022does}; this is called the max-margin solution. 
We now show that for a small enough leftover group, large enough scaling factor $B$ and dimension $d$ of the covariates, max-margin solutions depend more on the shortcut feature than the stable feature:
\newcommand{\scalingpropformalmain}{
Let $\mbw^*$ be the max-margin predictor
 on $n$ training samples from \cref{eq:sim-example} 
 with a leftover group of size $k$.
There exist constants $C_1, C_2, N_0 > 0$ such that
\begin{align}
\forall \, n  > N_0,  \qquad\quad
 \forall \,\, \text{ integers } k \in \left(0,\frac{n}{10}\right) , \qquad\quad 
 \forall \,\, d \geq C_1 k\log (3n),
    \qquad\quad 
\forall \,\, B > C_2 \sqrt{\frac{d}{k}},
\end{align}
with probability at least $1-\nicefrac{1}{3n}\,\,$ over draws of the training data, it holds that {$\,\, {{B \mbw_z^*} > {\mbw_y^*}}$}.}
\begin{thm}\label{thm:bad-case}
\scalingpropformalmain{}
\end{thm}
\vspace{5pt}
The size of the leftover group $k$ concentrates around $(1-\rho)n$ because each sample falls in the leftover group with probability $(1-\rho)$.
Thus, for $\rho>0.9$, that is for a strong enough shortcut, the condition in \cref{thm:bad-case} that $k <  \nicefrac{n}{10}$ will hold with probability close to $1$; see \cref{appsec:intuition} for more details.

The proof is in \cref{appsec:scalingprop}.
The first bit of intuition is that using the shortcut can have lower norm because of the scaling factor $B$. 
Using the shortcut only, however, misclassifies the leftover group. The next bit of intuition is that using noise from the leftover group increases margins in one group at a rate that scales with the dimension $d$, while the cost in the margin for the other group only grows as $\sqrt{d}$. This trade-off in margins means the leftover group can be correctly classified using noise without incorrectly classifying the shortcut group. 
The theorem then leverages convex duality to show that this type of classifier that uses the shortcut and noise has smaller $\ell_2$-norm than any linear classifier that uses the stable feature more.

The way the margin trade-off in the proof works is by constructing a linear classifier whose weights on the noise features are a scaled sum of the product of the label and the noise vector in the leftover group: for a scalar $\gamma$, the weights 
${\mbw_{e}= \gamma \textstyle\sum_{i\in S_{\text{leftover}}} \mby_i \mbdelta_i}$.
The margin change on the $j$th training sample from using these weights is $\mby_j \mbw_e^\top \mbdelta_j$.
For samples in the shortcut group, the margin change looks like a sum of mean zero independent and identically distributed variables; the standard deviation of this sum grows as $\sqrt{d}$. 
For samples in the leftover group, the margin change is the sum of mean one random variables; this sum grows as $d$ and its standard deviation grows as $\sqrt{d}$. The difference in mean relative to the standard deviation is what provides the trade-off in margins.

We now discuss three implications of the theorem.

\keypoint{First, \cref{thm:bad-case} implies that the leftover group sees worse than random accuracy ($0.5$).}
To see this, note that for samples in the leftover group the margin $\mby (\mbw^*)^\top\mbx = \mbw_y^* - B \mbw_z^* + (\mbw_e^*)^\top \mby \mbdelta$ is a Gaussian random variable centered at a negative number $\mbw_y^* - B \mbw_z^*$.
Then, with $\Phi_e$ as the CDF of the zero-mean Gaussian random variable $(\mbw_e^*)^\top \mbdelta$, accuracy in the test leftover group is 
\[p(\mby (\mbw^*)^\top\mbx \geq 0 \g \mby \neq \mbz) = p[(\mbw_e^*)^\top \mbdelta > -(\mbw_y^* - B \mbw_z^*)] =  1 - \Phi_e(-(\mbw_y^* - B \mbw_z^*)) \leq 0.5.\]

Second, the leftover group in the training data is overfit in that the contribution of noise in prediction ($|(\mbw_e^*)^\top \mbdelta |$) is greater than the contribution from the stable and shortcut features.
Formally, in the training leftover group, $\mbw^*_y - B\mbw^*_z<0$.
Then, due to max-margin property, 
\[\mbw_y^* - B \mbw_z^* + (\mbw_e^*)^\top \mby_i \mbdelta_i> 1 \implies (\mbw_e^*)^\top \mby_i \mbdelta_i\geq 1 - (\mbw_y^* - B \mbw_z^*) > |\mbw_y^* - B \mbw_z^*|.\]

Third, many works point to overparameterization as one of the causes behind shortcut learning \citep{sagawa2019distributionally,nagarajan2020understanding,wald2022malign}, but in the setup in \cref{fig:lin-erm-failure}, the linear model has fewer parameters than samples in the training data.
In such cases with non-overparameterized linear models, the choice of \derm{} is typically not questioned, especially when a feature exists that linearly separates the data.
\keypoint{\Cref{corr:noverparam} formally shows shortcut learning for non-overparameterized linear models.
In words, \derm{} --- that is vanilla logistic regression trained with gradient-based optimization --- can yield models that rely more on the shortcut feature \textit{even without overparameterization}.}
\begin{corr}\label{corr:noverparam}
For all $n > N_0$ --- where the constant $N_0$ is from \cref{thm:bad-case} --- with scalar $\tau\in (0,1)$ such that the dimension of $\mbx$ is $d=\tau n < n$,   
for all integers $k < n \times  \min \left\{\frac{1}{10}, \frac{\tau}{C_1 \log 3n}\right\},$ 
a linear model trained via  \derm{} yields a predictor $\mbw^*$ such that
 $ {B \mbw_z^*} > {\mbw_y^*}$.
\end{corr}
If \derm{} produces models that suffer from shortcut learning even without overparameterization, its implicit inductive bias toward max-margin classification is inappropriate for perception tasks in the presence of shortcuts.
Next, we study inductive biases more suited to perception tasks.

\section{Toward inductive biases for perception tasks with shortcuts}
\label{sec:vdm_validate}
The previous section 
formalized how \derm{} solutions, due to the max-margin inductive bias,  rely on the shortcut and noise to minimize loss on training data even in the presence of a different zero-population-risk solution. 
Are there inductive biases more suitable for perception tasks?

Given a perfect stable feature $s(\mbx)$ for a perception task, in that 
for a function $d$ when $\mby = d \circ s(\mbx)$, one can achieve margin $b \in (0,\infty)$ uniformly on all samples by predicting with the stable $b \cdot d\circ s(\mbx)$.
In contrast, max-margin classifiers allow for disparate margins as long as the smallest margin crosses $1$, meaning that it does not impose uniform margins.
The cost of allowing disparate margins is the preference for shortcuts even without overparamterization (\cref{corr:noverparam}). 
In the same setting however, any uniform-margin classifier for the linear perception task (\cref{eq:sim-example}) relies only on the stable feature:
\newcommand{\bumpylossthm}{
Consider $n$ samples of training data from \gls{dgp} in \cref{eq:sim-example} with $d<n$.
Consider a linear classifier $f_\theta(\mbx)=\mbw^\top \mbx$ such that for all samples in the training data $\mby_i \mbw^\top \mbx_i = b$ for any $b\in (0,\infty)$. 
With probability 1  over draws of samples, $\mbw=[0, b, 0^{d-2}].$
}

\begin{thm}\label{thm:bumpylossopt}
\bumpylossthm{}
\end{thm}

\Cref{thm:bumpylossopt} shows that uniform-margin classifiers only depend on the stable feature, standing in contrast with max-margin classifiers which can depend on the shortcut feature (\cref{thm:bad-case}).
The proof is in \cref{appsec:bumpy-loss}.
\keypoint{Thus, inductive biases toward uniform margins are better suited for perception tasks.} 
Next, we identify several ways to encourage uniform margins.

\paragraph{\glsreset{cc}\Gls{cc}.}
To produce uniform margins with gradient-based optimization, we want the loss to be minimized at uniform-margin solutions and be gradient-optimizable.
We identify a variety of losses that satisfy these properties, and we call them \gls{cc} losses.
{\gls{cc} losses have the property that per-sample loss monotonically decreases for margins until a threshold then increases for margins beyond it.}
In turn, minimizing loss then encourages all margins to move to the threshold.

Mechanically, when models depend more on shortcuts than the stable feature during training, margins on samples in the shortcut group will be larger than those in the leftover group;
see the right panel in \cref{fig:lin-sim-example-pergroup} where the train loss in the shortcut group is lower than the leftover group indicating that the margins are smaller in the leftover group. This difference is margins is a consequence of the shortcut matching the label in one group and not the other, thus, encouraging the model to have similar margins across all samples pushes the model to depend less on the shortcut.
In contrast, vanilla log-loss can be driven to zero in a direction with disparate margins across the groups as long as the margins on all samples go to $\infty$.
We define  \gls{cc} losses for a model $\ftheta$ with the margin on a sample $(\mbx, \mby)$ defined as $\mby \ftheta(\mbx)$.

\begin{wrapfigure}[12]{R}{0.67\textwidth}
\vspace{-15pt}
\small
    \hspace{-10pt}
    \includegraphics[width=0.68\textwidth]{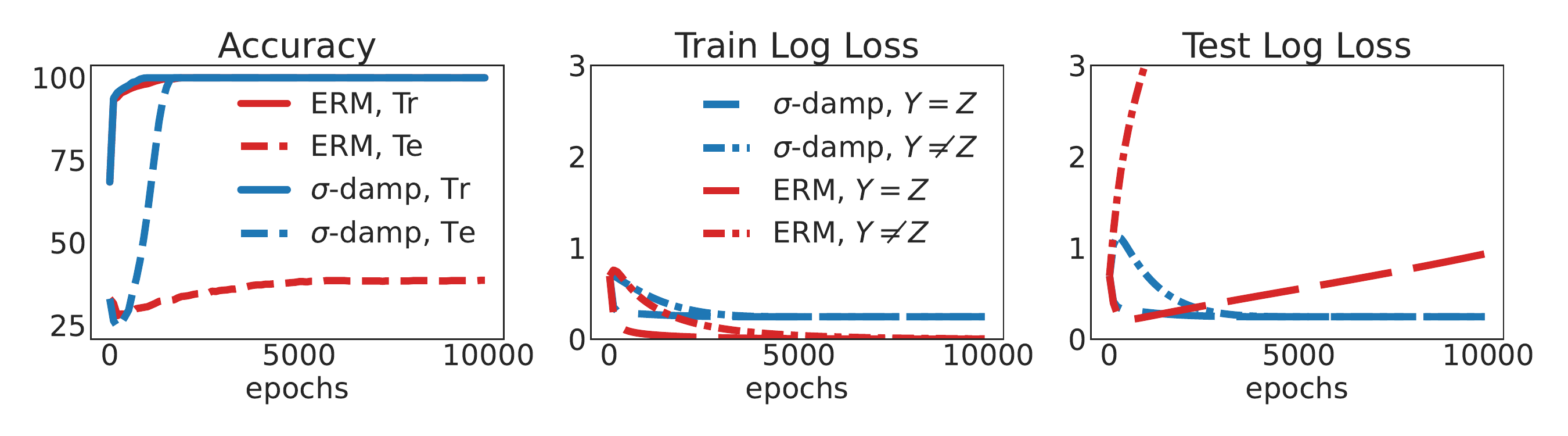}
    \vspace{-10pt}
    \caption{
    \small
Using $\sigma$-damped log-loss yields linear models that depend on the perfect stable feature to achieve near perfect test accuracy.
The middle panel shows that $\sigma$-damping maintains similar margins in the training shortcut and leftover groups unlike unconstrained log-loss, and the right panel shows \sigdamp{} achieves better leftover test-loss.
    }
    \label{fig:sigma-damp}
\vspace{-10pt}
\end{wrapfigure}

As the first \gls{cc} loss, we develop the  $\sigma$-damped log-loss: we evaluate log-loss on a margin multiplied by a monotonically decreasing function of the margin.
In turn, the input to the loss increases with the margin till a point and then decreases.
For a temperature $T$ and sigmoid function $\sigma$, the $\sigma$-damped loss modifies the model output $\ftheta$ and plugs it into log-loss:
\begin{align*}
\ell_{\text{\sigdamp{}}}(\mby,\ftheta) =
\ell_{\text{log}}\left(
        \mby 
        \left( 
        1 -\sigma\left(
        		\frac{\mby\ftheta}{T}
		        \right)
        \right)
        \ftheta
    \right)
\end{align*}

For large margin predictions  $\mby\ftheta > 0$, the term $ 
        1 -\sigma\left(\nicefrac{\mby\ftheta(\mbx)}{T}\right)$ damps down the input to log-loss.
The largest the input to $\ell_{\text{log}}$ can get is $0.278T$, found by setting the derivative to zero, thus lower bounding the loss.
As log-loss is a decreasing function of its input, the minimum of $\ell_{\text{\sigdamp{}}}$ occurs when the margin is $0.278T$ on all samples.
To demonstrate empirical advantage, we compare standard log-loss to $\sigma$-damped loss on \cref{eq:sim-example}; see \cref{fig:sigma-damp}.
The left panel of figure \cref{fig:sigma-damp} shows that test accuracy is better for \sigdamp.
The middle and right panels shows the effect of controlling margins in training, where losses on shortcut and leftover groups hover at the same value.

Second, we design the \stitch{} loss, which imitates log-loss when $\mby \ftheta(\mbx) < u$ and penalizes larger margins ($\mby\ftheta > u$) by negating the sign of $\mby \ftheta(\mbx)$:
\begin{align}\label{eq:stitch}
\begin{split}
\ell_{\text{\stitch}}  = \ell_{log}\left(\,\mathbf{1}[\mby\ftheta(\mbx) \leq u] \mby\ftheta(\mbx) \right.\,
  + \,\left.\mathbf{1}[\mby\ftheta(\mbx) > u] (2u-\mby\ftheta(\mbx))\,\right)
\end{split}
\end{align}

As the third \gls{cc} loss, we directly penalize large margins via a $\log$-penalty:
\begin{align}
	\ell_{\texttt{marg-log}} = \ell_{log}(\mby \ftheta(\mbx)) + \lambda \log\left(1 + |\ftheta(\mbx)|^2\right)
\end{align}

The fourth \gls{cc} loss controls margins by penalizing $|\ftheta(\mbx)|^2$:
\begin{align}
\ell_{\textsc{sd}} = \ell_{log}(\mby \ftheta(\mbx)) + \lambda |\ftheta(\mbx)|^2\label{eq:sd}
\end{align}
This last penalty was called \gls{sd} by \citet{pezeshki2021gradient}, who use it as a way to decouple learning dynamics in the \gls{ntk} regime.
Instead, from the lens of \gls{cc},  \gls{sd}  mitigates shortcuts in \cref{eq:sim-example} because it encourages uniform margins, even though \gls{sd} was originally derived from different principles, as we discuss in \cref{sec:related}.
In \cref{appsec:exps-synthetic-losses}, we plot all \gls{cc} losses and show that \gls{cc} improves over \derm{} on the linear perception task; see \cref{fig:lin-sigstitch,fig:lin-sd,fig:lin-sd-log}.
We also run \gls{cc} on a neural network and show that while \derm{} achieves test accuracy worse than random chance, \gls{cc} achieves $100\%$ test accuracy; see \cref{fig:sigstitch,fig:sigdamp,fig:sd-log,fig:sd} in \cref{cc-on-nn}.

\section{Vision and language experiments}\label{sec:experiments}

\newacronym{sd}{sd}{spectral decoupling}

We evaluate \gls{cc} on common datasets with shortcuts: Waterbirds, CelebA \citep{sagawa2019distributionally}, and Civilcomments \citep{pmlr-v139-koh21a}.
First, \gls{cc} always improves over \derm{.}
Then, we show that \gls{cc} performs similar to or better than two-stage \sms{} like \gls{jtt} \citep{liu2021just} and \gls{cnc} \citep{pmlr-v162-zhang22z} in traditional evaluation settings where group annotations are available in the validation data. 
Finally, we introduce a more challenging setting that only provides class labels in training and validation, called the \textbf{nuisance-free setting}.
In contrast to the traditional setting that always assumes validation group annotations, the nuisance-free setting does not provide group annotations in either training or in validation.
In the nuisance-free setting, \gls{cc} outperforms \gls{jtt} and \gls{cnc}, even though the latter are supposed to mitigate shortcuts without knowledge of the groups.

\paragraph{Datasets.}
We use the Waterbirds and CelebA datasets from \citet{sagawa2019distributionally} and the CivilComments dataset from \citet{borkan2019nuanced,pmlr-v139-koh21a}.
In Waterbirds, the task is to classify images of a waterbird or landbird, and the label is spuriously correlated with the image background consisting of land or water.
There are two types of birds and two types of background, leading to a total of 4 groups defined by values of $y,z$.
In CelebA \citep{liu2015deep,sagawa2019distributionally}, the task is to classify hair color of
celebrities as blond or not.
The gender of the celebrity is a shortcut for hair color.
There are two types of hair color and two genders in this dataset, leading to a total of 4 groups defined by values of $y,z$.
In CivilComments-WILDS \citep{borkan2019nuanced,pmlr-v139-koh21a}, the task is to classify whether an online
comment is toxic or non-toxic, and the label is spuriously correlated with mentions of certain demographic
identities.
There are $2$ labels and $8$ types of the shortcut features, leading to $16$ groups.

\begin{wraptable}[18]{R}{0.52\textwidth}
\small
\centering
\vspace{-10pt}
\begin{tabular}{cccc}
\toprule
 & CelebA & WB & Civil \\
 \midrule
 \acrshort{erm} &  $72.8\pm 9.4$ & $70.8\pm 2.4$ & $60.1\pm 0.4$  \\
\acrshort{cnc} &  ${81.1\pm 0.6}$ & $68.0\pm 1.8$ & ${68.8\pm 0.2}$  \\
\gls{jtt}
        &  $75.2\pm 4.6$ 
        & $71.7\pm 4.0$
        & ${69.9\pm 0.4}$  \\
\midrule
\texttt{marg-log}
     &  ${82.8 \pm 1.1}$ 
     & ${78.2 \pm 1.9}$ 
     & ${68.4\pm 1.8}$  \\
\sigdamp{} 
    &  ${79.4\pm 0.6}$ 
    & ${78.6\pm 1.1}$ 
    & ${69.6\pm 0.4}$  \\
\textsc{sd} 
    &  ${81.4\pm 2.5}$ 
    & ${80.5\pm 1.4}$ 
    & ${69.9\pm 1.1}$  \\
\stitch 
    &  ${81.1\pm 2.2}$ 
    & ${75.9\pm 3.4}$ 
    & ${67.8\pm 2.8}$  \\
\bottomrule
\end{tabular}
      \caption{\small
Mean and standard deviation of test worst-group accuracies over two seeds for \derm{}, \gls{jtt}, \gls{cnc}, \sigdamp{,} \stitch{}, \gls{sd}, and \texttt{marg-log}.
Every \gls{cc} method outperforms \derm{} on every dataset.
On Waterbirds, \gls{cc} outperforms \gls{jtt} and \gls{cnc}.
On CelebA, \gls{sd}, ${\texttt{marg-log}}$, and \stitch{} beat \gls{jtt} and achieve similar or better performance than \gls{cnc}.
On CivilComments, \sigdamp{} and \gls{sd} beat \gls{cnc} and achieve similar performance to \gls{jtt}.
}
      \label{tab:results}
\vspace{-10pt}
\end{wraptable}

\paragraph{Metrics, model selection, and hyperparameters.}
We report the worst-group test accuracy for each method.
The groups are defined based on the labels and shortcut features. The more a model depends on the shortcut, the worse the worst-group error.
Due to the label imbalance in all the datasets, we use variants of \sigdamp{}, \stitch{}, \textsc{marg-log}, and \gls{sd} with class-dependent hyperparameters; see \cref{appsec:training-details}.
For all methods, we use the standard Adam optimizer \citep{kingma2014adam} and let the learning rate and weight decay hyperparameters be tuned along with the method's hyperparameters.
We first report results for all methods using validation worst-group accuracy to select method and optimization hyperparameters and early stop.
For both \gls{jtt} and \gls{cnc}, 
this is the evaluation setting that is used in existing work \citep{liu2021just,idrissi2021simple,pmlr-v162-zhang22z}.
Finally, in the nuisance-free setting where no group annotations are available, we select hyperparameters using label-balanced average accuracy.
\Cref{appsec:exps} gives further details about the training, hyperparameters, and experimental results.

\subsection{\Gls{cc} mitigates shortcuts in the default setting}

Here, we experiment in the standard setting from \citet{liu2021just,idrissi2021simple,pmlr-v162-zhang22z} and use validation group annotations to tune hyperparameters and early-stopping.

\paragraph{\Gls{cc} improves over \derm{.}}
We compare \gls{cc} to \derm{} on CelebA, Waterbirds, and Civilcomments.
\Cref{tab:results} shows that every \gls{cc} method achieves higher test worst-group accuracy than \derm{} on all datasets.
\Derm{} achieves a mean test worst-group accuracy of $70.8\%, 72.8\%$ and $60.1\%$ on Waterbirds, CelebA, and Civilcomments respectively.
Compared to \derm{}, \gls{cc} methods provide a $5-10\%$ improvement on Waterbirds, $7-10\%$ improvement on CelebA, $7-10\%$ improvement on Civilcomments.
These improvements show the value of inductive biases more suitable for perception tasks.

\begin{figure}[t]
\small
\centering
\includegraphics[width=0.9\textwidth]{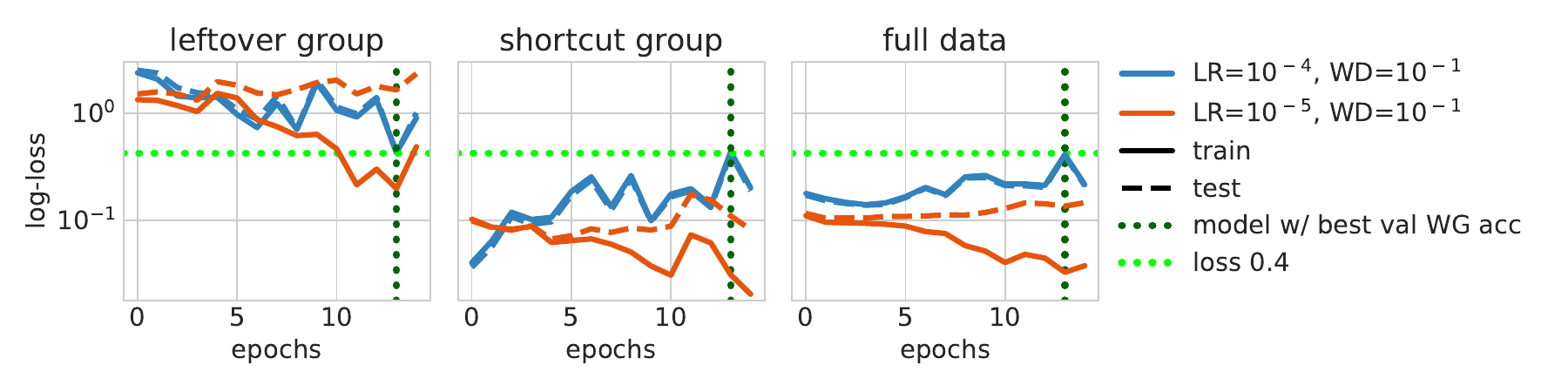}
\vspace{-5pt}
\caption{ 
\small
Loss curves of \derm{} on CelebA for two combinations of \gls{lr} and \gls{wd}.
The combination with the larger learning rate (blue) achieves $72.8\%$ test worst-group accuracy, beating the other combination by $20\%$.
The model that achieves the best validation (and test) worst-group accuracy is the one at epoch $13$ from the blue run.
This model achieves similar loss in both groups and the full data model suggesting that large \gls{lr} and \gls{wd} controls margins from exploding (higher training loss in all panels) and avoids systematically smaller margins in the leftover group compared to the shortcut group.
}
	\label{fig:Celeba-lr-wd}
 \vspace{-5pt}
\end{figure}

\paragraph{Large \gls{lr} and \gls{wd} may imitate \gls{cc} in \acrshort{erm}.}
\Derm{'s} performance varies greatly across different values of \gls{lr} and \gls{wd} on, for instance, CelebA: the test worst-group accuracy improves by more than $20$ points over different \gls{lr} and \gls{wd} combinations.
Why does tuning \gls{lr} and \gls{wd} yield such improvements?
We explain this phenomenon as a consequence of instability in optimization induced by large \gls{lr} and \gls{wd} which prevents the model from maximizing margins and in turn can control margins.
\Cref{fig:Celeba-lr-wd} provides evidence for this explanation by comparing \derm{'s} loss curves for two \gls{lr} and \gls{wd} combinations.

The blue loss curves in \cref{fig:Celeba-lr-wd} correspond to the run with the larger \gls{lr} and \gls{wd} combination.
The model that achieves the best validation (and test) worst-group accuracy over all combinations of hyperparameters for \derm{}, including those not in the plot, is the one at epoch $13$ on the blue curves.
This model achieves similar train and test losses ($\approx 0.4$) and thus similar margins in the shortcut group, the leftover group, and the whole dataset.
The red curves stand in contrast where the lower \gls{lr} results in the leftover group having higher training and test losses, and therefore smaller margins, compared to the shortcut group.
These observations together support the explanation that \derm{} with large \gls{lr} and \gls{wd} mitigates shortcuts when controlling margins like \gls{cc}.

\paragraph{\Gls{cc} performs as well or better than two-stage \sms{}.}

\glsreset{jtt}
\glsreset{cnc}
Two-stage shortcut mitigating methods like \gls{cnc} and \gls{jtt} aim to mitigate shortcuts by using a model trained with \derm{} to approximate group annotations.
They rely on the assumption that a model trained via  \derm{} either predicts with the shortcut feature (like background in Waterbirds) or that the model's representations separate into clusters based on the shortcut feature.
The methods then approximate group annotations using this \derm{-trained} model and use them to mitigate shortcut learning in a second predictive model.
\Gls{jtt} upweights the loss on the approximate leftover group and \gls{cnc} uses a contrastive loss to enforce the model's representations to be similar across samples that have the same label but different approximate group annotations.
\Cref{appsec:exps-twostage} gives details.

\Cref{tab:results} compares \gls{cc} to \gls{jtt} and \gls{cnc} on Waterbirds, Celeba, and CivilComments.
On CelebA, \gls{sd}, \texttt{marg-log}, and \stitch{} perform similar to \gls{cnc} while all \gls{cc} techniques outperform \gls{jtt}.
On Waterbirds, all \gls{cc} methods outperform \gls{jtt} and \gls{cnc}.
On CivilComments, \sigdamp{} and \gls{sd} perform similar to \gls{jtt} and outperform \gls{cnc}. 
\Gls{cnc}'s performance on Waterbirds differs from \citet{pmlr-v162-zhang22z} because their reported performance requires unique large \gls{wd} choices (like \gls{wd} set to $1$) to build a first-stage model that relies most on the shortcut feature without overfitting to the training data.

\paragraph{\Gls{cc} is faster than \gls{jtt} and \gls{cnc}.}
\Gls{cc} takes the same time as \derm{}, taking around $1,20$ and $60$ minutes per epoch for Waterbirds, CelebA, and CivilComments respectively on an RTX8000 GPU.
In contrast, on average over runs, \gls{jtt} takes around $6,80, 120$ minutes per epoch and \gls{cnc} takes around $8, 180, 360$ minutes per epoch.
Thus, \gls{cc} performs as well or better than \gls{jtt} and \gls{cnc} while being simpler to implement and computationally cheaper.

\subsection{\Gls{cc} mitigates shortcuts in the nuisance-free setting}

\begin{wraptable}[15]{R}{0.47\textwidth} 
\small
\centering
\vspace{-12pt}
\begin{tabular}{cccc}
\toprule
& CelebA & WB & Civil \\
\midrule
\acrshort{erm}
        & $57.5\pm 5.8$ 
        & $69.1\pm 2.1$ 
        & $60.7\pm 1.5$ \\
\gls{cnc} 
        & $67.8\pm 0.6$ 
        & $60.0\pm 8.0$
        & $61.4\pm 1.9$ \\
\gls{jtt}
        & $53.3\pm 3.3$ 
        & $71.7\pm 4.0$ 
        & $53.4\pm 2.1$ \\	
\midrule
\texttt{marg-log} 
        & $74.2\pm 1.4$ 
        & ${77.9\pm 0.3}$ 
        & ${66.8\pm 0.2}$ \\
\sigdamp{} 
        & ${70.8\pm 0.3}$ 
        & ${74.8 \pm 1.6}$ 
        & ${65.6\pm 0.2}$ \\
\gls{sd} 
        & ${70.3\pm 0.3}$ 
        & ${78.7\pm 1.4}$ 
        & ${67.8\pm 1.3}$ \\
\stitch{}
        & ${76.7\pm 0.6}$ 
        & $74.5\pm 1.2$ 
        & ${66.0\pm 1.0}$ \\
\bottomrule
\end{tabular}
\vspace{-3pt}
\caption{
\small
Average and standard deviation of test worst-group accuracy over two seeds of \gls{cc}, \derm{,} \gls{jtt}, and \gls{cnc} in {the nuisance-free setting}.
Hyperparameter selection and early stopping use label-balanced average accuracy.
All \gls{cc} methods outperform \derm{}, \gls{jtt}, and \gls{cnc} on all datasets.
}
\label{tab:pure}
\vspace{-15pt}
\end{wraptable}

Work like \citep{liu2021just,pmlr-v162-zhang22z} crucially require validation group annotations 
because these methods push the work of selecting models for mitigating shortcuts to validation.
Determining shortcuts itself is a laborious manual process, which means group annotations will often be unavailable.
Further, given a perfect stable feature that determines the label and a shortcut that does not, only models that rely on the stable feature more than the shortcut can achieve the highest validation accuracy.
Thus, we introduce a more challenging setting that only provides class labels in training and validation, called the \textbf{nuisance-free setting}.
In the nuisance-free setting, models are selected based on label-balanced average accuracy: the average of the accuracies over samples of each class.

\Cref{tab:pure} reports test worst-group (WG) accuracy in the nuisance-free setting.
\keypoint{On all the datasets, every \gls{cc} outperforms \derm{}, \gls{jtt}, and \gls{cnc}.} 
On average, the \gls{cc} methods close at least $61\%$ of the gap between \derm{} in the nuisance-free setting and the best performance in \cref{tab:results} on every dataset.
\keypoint{In contrast, \gls{cnc} and \gls{jtt} sometimes perform worse than \derm{}.
}

\section{Related work}\label{sec:related}

A large body of work tackles shortcut learning under different assumptions \citep{arjovsky2019invariant,wald2021calibration,krueger2020out,creager2021environment,veitch2021counterfactual,puli2022outofdistribution,heinze2017conditional,belinkov2017synthetic}.
A different line of work focuses on learning in neural networks in idealized settings \citep{yang2019fine,ronen2019convergence,jo2017measuring,baker2018deep,saxe2013exact,gidel2019implicit,advani2020high}.

\citet{shah2020pitfalls} study simplicity bias \citep{valle2018deep} and show that neural networks provably learn the linear function over a non-linear one, in the first epoch of training.
In a similar vein, \citet{hermann2020shapes} show that neural networks can prefer a linearly-decodable feature over a non-linear but more predictive feature, and
\citet{scimeca2021shortcut} make similar observations and use loss landscapes to empirically study which features are easier to learn.
Simplicity bias alone only describes neural biases early in training and does not explain why more predictive stable features are not learned later.
Unlike simplicity bias which focuses on linear versus non-linear features, max-margin bias is the reason \derm{} prefers one linear feature, the shortcut, over another, the stable feature, like in the synthetic experiment in \cref{sec:method}.

While \citet{pezeshki2021gradient} allow for perfect features, they hypothesize that shortcut learning occurs because when one feature is learned first, other features are gradient-starved and are not learned as well. They focus on a special setting where feature representations for different samples have inner product equal to a small constant to show that models can depend more on the imperfect feature than the perfect feature.
In this special setting, they show that penalizing the magnitudes of what we call the margin mitigates shortcuts; this method is called \glsreset{sd}\gls{sd}.
However, as we show in \cref{appsec:spectral-decoupling}, the assumption in Lemma 1 \citep{pezeshki2021gradient} is violated when using a linear model to classify in the simple linear \gls{dgp} in \cref{eq:sim-example}.
However, \gls{sd} on a linear model mitigates shortcuts in the \gls{dgp} in \cref{eq:sim-example}; see \ref{appsec:spectral-decoupling}.
Thus, the theory in \citet{pezeshki2021gradient} fails to not explain why \gls{sd} works for \cref{eq:sim-example}, but the uniform-margin property explains why all the \gls{cc} losses, including \gls{sd}, mitigate shortcuts.

\citet{nagarajan2020understanding} consider tasks with perfect stable features and formalize geometric properties of the data that make max-margin classifiers give non-zero weight to the shortcut feature ($\mbw_z>0$).
In their set up, the linear models are overparameterized and it is unclear when $\mbw_z>0$ leads to worse-than-random accuracy in the leftover group because they do not separate the model's dependence on the stable feature from the dependence on noise.
See \cref{fig:wz-pos-case} for an example where $\mbw_z>0$ but test accuracy is $100\%$.
In contrast to \citet{nagarajan2020understanding}, \cref{thm:bad-case} gives a family of \glspl{dgp} where the leftover group accuracy is worse than random, even without overparameterization.
\citet{ahuja2021invariance} also consider linear classification with \derm{} with a perfect stable feature and conclude that \derm{} learns only the stable feature because they assume no additional dimensions of noise in the covariates.
We develop the necessary nuance here by including noise in the problem and showing \derm{} depends on the shortcut feature even without overparameterization.

\citet{sagawa2020investigation} and \citet{wald2022malign} both consider overparameterized settings where the shortcut feature is informative of the label even after conditioning on the stable feature.
In both cases, the Bayes-optimal predictor also depends on the shortcut feature, which means their settings do not allow for an explanation of shortcut dependence in examples like \cref{fig:lin-erm-failure}.
In contrast, we show shortcut dependence occurs even in the presence of a perfect stable feature and without overparameterization.
\citet{li2019towards,pezeshki2022multi} focus on relative feature complexity and discuss the effects of large  \gls{lr} on which features are learned first during training, but do not allow for perfect features.
\citet{idrissi2021simple} empirically find that tuning \gls{lr} and \gls{wd} gets \derm{} to perform similar to two-stage \sms{} like \gls{jtt} \citep{liu2021just}.
We view the findings of \citep{idrissi2021simple} through the lens of \gls{cc} and explain how large \gls{lr} and \gls{wd} approximate \gls{cc} to mitigate shortcuts; see \cref{sec:experiments}.

\Gls{cc} is related to but different from methods proposed in
\citet{liu2017sphereface,cao2019learning,kini2021label}.
These works normalize representations or the last linear layers and linearly transform the logits to learn models with better margins under label imbalance.
Next, 
methods like \gls{lff} \citep{nam2020learning}, \gls{jtt} \citep{liu2021just}, and \gls{cnc} \citep{pmlr-v162-zhang22z} build two-stage procedures to avoid shortcut learning without group annotations in training. They assume that \derm{} produces models that depend more on the shortcut and select hyperparamters of the two stage process using validation group annotations. In the nuisance-free setting where there are no validation group annotations, the performance of these methods can degrade below that of \derm{}. In contrast, better characterizing the source of shortcut learning in perceptual problems leads to \gls{cc} methods that are not as reliant on validation group annotations (see nuisance-free results in \Cref{sec:experiments}).
\keypoint{Without any group annotations, encouraging uniform margins via \gls{cc} mitigates shortcuts better than \gls{jtt} and \gls{cnc}.}

\citet{soudry2018implicit} characterize the inductive bias of gradient descent to converge in direction to max-margin solutions when using exponentially tailed loses; \citet{wang2021implicit,wang2022does} then prove similar biases toward max-margin solutions for Adam and RMSProp.
\citet{ji2020gradient} show that for general losses that decrease in $\mby \ftheta(\mbx)$, gradient descent has an inductive bias to follow the $\ell_2$-regularization path.
All these inductive biases prefer shortcuts if using them leads to lower loss within an $\ell_2$-norm-budget.
\Gls{cc} provides a different inductive bias toward producing the same margin on all samples, which means gradient descent veers models away from imperfect shortcuts that lead to disparity in network outputs.
Such inductive biases are suitable for tasks where a feature determines the label ($h(\mbx) = \mby$).

\section{Discussion}
We study why \derm{} --- gradient-based optimization of log-loss --- yields models that depend on the shortcut even when the population minimum of log-loss is achieved by models that depend only on the stable feature.
By studying a linear task with perfect stable features, we show that \derm{'s}  preference toward shortcuts sprouts from an inductive bias toward maximizing margins.
Instead, inductive biases toward uniform margins improve dependence on the stable feature and can be implemented via \gls{cc}.
\Gls{cc} improves over \derm{} on a variety of perception tasks in vision and language without group annotations in training, and is competitive with more expensive two-stage \sms{}.
In the nuisance-free setting, where even validation group annotations are unavailable, \gls{cc} outperforms all the baselines.
{The performance that \Gls{cc} yields demonstrates that changing inductive biases can remove the need for expensive \sms{} in perception tasks.}
 
Without overparameterization, uniform-margin classifiers are unique and learn stable features only, while max-margin classifiers can depend more on shortcuts.
With overparameterization, max-margin classifiers are still unique but uniform-margin solutions are not which necessitates choosing between solutions. 
The experiments in \cref{sec:experiments} suggest that choosing between uniform-margin classifiers with penalties like $\ell_2$ improves over max-margin classifiers with $\ell_2$: all experiments use overparameterized models trained with weight decay and \gls{cc} outperforms \derm{.}
Further, our experiments suggest that uniform-margin classifiers are insensitive to the \gls{wd} and \gls{lr} choices, unlike max-margin classifiers; \cref{appsec:additional} shows that \gls{cc} achieves high performance for all \gls{lr} and \gls{wd} choices but \acrshort{erm} requires tuning.

\Cref{thm:bad-case} also explains how balancing may or may not improve dependence on the stable features. 
For example, a weighting-based approach produces the same max-margin solution as \derm{} \citep{sagawa2020investigation,rosset2003margin}, but subsampling leads to a different solution that could depend less on the shortcut.
For the latter however, models are more prone to overfitting on the smaller subsampled dataset.
Similar observations were made in \cite{sagawa2020investigation} but this work extends the insight to tasks with perfect stable features.
Comparing \acrshort{erm} and \gls{cc} on subsampled data would be fruitful.

Any exponentially tailed loss when minimized via gradient descent converges to the max-margin solution in direction \citep{soudry2018implicit}.
Thus, \cref{thm:bad-case} characterizes shortcut learning for any exponentially-tailed loss.
However, losses with decreasing polynomial tails --- for example, $\ell(a) = \frac{1}{1 + a^K}$ for some $K>0$ --- do not converge to the max-margin classifier.
One future direction is to show shortcut-dependence results like \cref{thm:bad-case} for polynomial-tailed losses, which in turn would mean that all common classification losses with a decreasing tail impose inductive biases unsuitable for perception tasks.

In the tasks we consider with perfect stable features, Bayes-optimal predictors rely only on the stable feature.
A weaker independence condition implies the same property of Bayes-optimal predictors even when $\mby$ is not determined by $s(\mbx)$: $\mby \indep (\mbx, \mbz) \g s(\mbx)$.
For example, in the CivilComments dataset a few instances have ambiguous labels \citep{xenos2022toxicity} meaning that there may not be a perfect stable feature.
Studying uniform margins and other inductive biases under this independence would be fruitful.

\bibliographystyle{unsrtnat}
\bibliography{ms}

\appendix

\onecolumn
\allowdisplaybreaks
\appendix

\section{Appendix: Proof of \Cref{thm:bad-case}, \Cref{corr:noverparam}, and \Cref{thm:bumpylossopt}}\label{appsec:scalingprop}

\subsection{Helper Lemmas}

\subsubsection{Bounding norms and inner products of isotropic random vectors.}
The main lemmas of this section are \cref{lem:inner-products} and \cref{lem:norm-conc}.
We will then use these two to  bound norms of sums of random vectors and inner products between the sum and a single random vector in \cref{lemma:concentration}.
We first list some facts from \citep{vershynin2018high} that we will use to bound the probability with which norms and inner products of Gaussian random vectors deviate far from their mean.\\

\begin{thmdef}
	(Sub-Gaussian norm) 
	For an r.v. $\mbx$, the sub-Gaussian norm, or $\psi_2$-norm, is
	\[\|\mbx\|_{\psi_2} = \inf\{t > 0, \E[\exp(\nicefrac{\mbx^2}{t^2})]\leq 2\}.\]
An r.v. is called sub-Gaussian if its $\psi_2$-norm is finite and for some fixed constant $c$
		\[
			p(|\mbx| > t) \leq 2 
		\exp(\nicefrac{-ct^2}{\|\mbx\|_{\psi_2}}
		).\]
\end{thmdef}
A Gaussian r.v. $\mbx\sim \cN(0, \sigma^2)$ has an $\psi_2$-norm of $G\sigma$ for a constant $G=\sqrt{\frac{8}{3}}$.\footnote{ $G=\sqrt{\frac{8}{3}}$.
This follows from:
\begin{align*}
\E_{\mbx\sim \cN(0,\sigma^2)}[\exp(\nicefrac{\mbx^2}{t^2})] & = 
\int_{-\infty}^\infty \frac{1}{\sigma\sqrt{2\pi}}\exp(-\nicefrac{x^2}{2\sigma^2})\exp(\nicefrac{x^2}{t^2}) dx = 
\int_{-\infty}^\infty \frac{1}{\sigma\sqrt{2\pi}}\exp\left(-x^2\frac{(t^2 - 2\sigma^2)}{2\sigma^2t^2}\right) dx
& 
\\
	& =	\frac{1}{\sigma\sqrt{2\pi}}\sqrt{\frac{\pi}{\frac{(t^2 - 2\sigma^2)}{2\sigma^2t^2}}} = \frac{1}{\sigma\sqrt{\pi}}\sqrt{\frac{\pi \sigma^2 t^2}{(t^2 - 2\sigma^2)}}
	 = \sqrt{\frac{t^2}{(t^2 - 2\sigma^2)}}
\\
\sqrt{\frac{t^2}{(t^2 - 2\sigma^2)}} \leq 2 & \implies  t^2  \leq 4(t^2 - 2\sigma^2) \implies  8 \sigma^2 \leq 3t^2 \implies \inf\{t : 8 \sigma^2 \leq 3t^2\} = \sqrt{\frac{8}{3}}\sigma.
\end{align*}	
\vspace{-25pt}
}\\

\begin{thmdef}
	(Sub-exponential norm) 
	For an r.v.  $\mbx$, the sub-exponential norm, or $\psi_1$-norm, is
	\[\|\mbx\|_{\psi_1} = \inf\{t > 0, \E[\exp(\nicefrac{|\mbx|}{t})]\leq 2\}.\]
A sub-exponential r.v. is one that has finite $\psi_1$-norm.\\
\end{thmdef}

\begin{lemma}\label{lem:products} (Lemma 2.7.7 from \citep{vershynin2018high})
Products of sub-Gaussian random variables $\mbx, \mby$ is a sub-exponential random variable with it's $\psi_1$-norm bounded by the product of the $\psi_2$-norm
\[\|\mbx \mby\|_{\psi_1} \leq \|\mbx\|_{\psi_2} \|\mby\|_{\psi_2}\]
\end{lemma}

\Cref{lem:products} implies that the product of two mean-zero standard normal  vectors is a sub-exponential random variable with $\psi_1$-norm less than $G^2$.\\

\begin{lemma}\label{lem:Bernstein} (Bernstein inequality, Theorem 2.8.2 \citep{vershynin2018high})
For i.i.d sub-exponential random variables $\mbx_1, \cdots, \mbx_d$,  for a fixed constant $c=\frac{1}{(2e)^2}$ and $K=\|\mbx_1\|_{\psi_1}$
\[p\left(\left|\sum_{i=1}^d \mbx_i\right| > t\right)\leq 2 \exp\left(-c
		\min\left\{
				\frac{t^2}{K^2d}, \frac{t}{K}\right
				\}
				\right)\]

\end{lemma}

Next, we apply these facts to bound the sizes of inner products between two unit-variance Gaussian vectors.\\

\begin{lemma}\label{lem:inner-products}(Bounds on inner products of Gaussian vectors)
Let $\mbu, \mbv$ be $d$-dimensional random vectors where each coordinate is an i.i.d standard normal r.v.
Then, for any scalar $\epsilon>0$ such that $\epsilon \leq G^2\sqrt{d}$,  for a fixed constant $c=\frac{1}{(2e)^2}$ 
\[p\left(\left|\mbu^\top \mbv\right| > \epsilon \sqrt{d}  \right)\leq 2 \exp\left(-c\frac{\epsilon^2}{G^4}\right).
\]
\end{lemma}
\begin{proof}
First, the inner product is $\mbu^\top\mbv = \sum_{i}^d \mbu_i \mbv_i$; it is the sum of products of i.i.d. standard normal r.v. ($\sigma=1$).
Then, by \cref{lem:products}, each term in the sum is a sub-exponential r.v. with $\psi_1$-norm bounded as follows:
\begin{align}
\label{eq:sub-exp-bound}
	K = \|\mbu_i \mbv_i\|_{\psi_1} \leq \|\mbu_i\|_{\psi_2}\|\mbu_i\|_{\psi_2} = G  \times G =  G^2.
\end{align}
We can apply Bernstein inequality  \cref{lem:Bernstein} to sub-exponential r.v. to the inner product and then upper bound the probability by replacing $K$ with the larger $G^2$ in \cref{eq:sub-exp-bound}
\[
p
\left(
	|\mbu^\top\mbv| > t
\right)
	\leq 
	2 \exp\left(-c
		\min\left\{
				\frac{t^2}{K^2d}, \frac{t}{K}\right
				\}
				\right) 
	\leq 2\exp\left(-c
		\min\left\{
				\frac{t^2}{G^4 d}, \frac{t}{G^2}\right
				\}
				\right)
\]

Substituting $t = \epsilon\sqrt{d}$ in the above gives us:

\[p\left(|\mbu^\top\mbv| > \epsilon\sqrt{d} \right)
	\leq 2\exp\left(-c
		\min\left\{
				\frac{\epsilon^2 d}{G^4 d}, \frac{\epsilon\sqrt{d}}{G^2}\right
				\}
				\right)
\]
Using the fact that $\epsilon \leq G^2 \sqrt{d}$ to achieve the minimum concludes the proof:
\[
	\epsilon \leq G^2 \sqrt{d} 
	\implies 
		\epsilon^2 \leq \epsilon G^2 \sqrt{d} 
	\implies 
		\frac{\epsilon^2}{G^4} 
	\leq
		\frac{\epsilon \sqrt{d}}{G^2} 
	\implies 
	\min\left\{
				\frac{\epsilon^2d}{G^4 d}, \frac{\epsilon\sqrt{d}}{G^2}\right
				\} 
	= 
		\frac{\epsilon^2}{G^4}
\]
\end{proof}

\begin{lemma}\label{lem:norm-conc}
Let $\mbx$ be a Gaussian vector of size $d$ where each element is a standard normal, meaning that $\|\mbx_i\|_{\psi_2}=G$.
Then, for any $t>0$ and a fixed constant $c=\frac{1}{(2e)^2}$ , the norm of the vector concentrates around $\sqrt{d}$ according to
\[
p\left(\left|\|\mbx\| - \sqrt{d}\right| > t\right) \leq 2 
	\exp(\nicefrac{-ct^2}{G^4}).
\]
\end{lemma}
\begin{proof}
Equation 3.3 from the proof of theorem 3.1.1 in \citep{vershynin2018high} shows that
\[p(|\|\mbx\| - \sqrt{d}| > t) \leq 2 \exp(\nicefrac{-ct^2}{\left(\max_i \|\mbx_i\|_{\psi_2}\right)^4}).\]
As $\mbx$ has i.i.d standard normal entries, $\max_i \|\mbx_i\|_{\psi_2} = G$, concluding the proof. 
\end{proof}

\subsubsection{Concentration of norms of sums of random vectors and their inner products}

This is the main lemma that we will use in proving \cref{thm:bad-case}.\\

 \begin{lemma}\label{lemma:concentration}
 Consider a set of vectors $V = \{\mbdelta_i\}$ where $\mbdelta_i\in \mathbf{R}^d$ of size $T_V\geq 1$ where each element of each vector is drawn independently from the standard normal distribution $\cN(0,1)$.
 Then,  for a fixed constant $c=\frac{1}{(2e)^2}$ and any $\epsilon \in (0,G^2\sqrt{d})$ with probability $ \geq 1 - 2 \exp(-\epsilon^2\frac{c}{G^4}) $
\begin{align}
\label{eq:lemma5-norm-bound}
	\left\|\frac{1}{\sqrt{T_V}}\sum_{i\in V} \mbdelta_i\right\| \leq \sqrt{d} + \epsilon
\end{align}
and with probability $ \geq 1 - 4 T_V \exp(-\epsilon^2\frac{c}{G^4}) $
\begin{align}
\label{eq:leftover-inner-product}
\forall \mbdelta_j \in V \quad  \left\langle \mbdelta_j,\sum_{i\in V} \mbdelta_i\right\rangle \geq d - 3\epsilon\sqrt{T_V d} 
\end{align}

Further, consider any set $U$ of vectors $U=\{\mbdelta_i\}$ of size $T_U$, where 
each vector also has coordinates drawn i.i.d from the standard normal distribution $\cN(0,1)$.
Then, with probability $ \geq 1 - 2T_u \exp(-\epsilon^2\frac{c }{G^4}) $
\begin{align}\label{eq:shortcut-inner-product}
  \forall \mbdelta_j \in U \quad \left|\left\langle \mbdelta_j , \sum_{i\in V} \mbdelta_i\right\rangle\right| \leq \epsilon \sqrt{T_V d},
\end{align}

By union bound, the three events above hold at once with a probability at least $1 - 2(2T_V + T_u + 1)\exp(-\epsilon^2\frac{c}{G^4})$.
\end{lemma}

\begin{proof}
We split the proof into three parts one each for \cref{eq:lemma5-norm-bound,eq:leftover-inner-product,eq:shortcut-inner-product}.
\paragraph{Proof of \cref{eq:lemma5-norm-bound}.}
As $\mbdelta$ is a vector of random i.i.d standard normal random variables, note that $\frac{1}{\sqrt{T_V}}\sum_i \mbdelta_i$ is also a vector of i.i.d standard normal random variables.
This follows from the fact that the sum of $T_V$ standard normal random variables is a mean-zero Gaussian random variable with standard deviation $\sqrt{T_V}$.
Thus dividing by the standard deviation makes the variance $1$, making it standard normal.

Then, applying \cref{lem:norm-conc} with $t=\epsilon$ gives us the following bound:
\[
p\left(
\left\|
	\frac{1}{\sqrt{T_V}}\sum_i \mbdelta_i
	\right\| > \sqrt{d} + \epsilon
\right)  
\leq 
	p
	\left(
	\left| \,\,
		\left\|
			\frac{1}{\sqrt{T_V}}\sum_i \mbdelta_i
	\right\| - \sqrt{d}
	\right| > \epsilon
\right) \leq 2 
	\exp(\nicefrac{-c\epsilon^2}{G^4})
\]

\paragraph{Proof of \cref{eq:leftover-inner-product} }
We split the inner product  into two cases: $T_V=1$ and $T_V \geq 2$.
\paragraph{Case $T_V=1$.}
First note that due to \cref{lem:norm-conc},
\[
\forall j \in V,  \qquad \quad p\left(
\|\mbdelta_j\| < \sqrt{d} - \epsilon
\right)  
\leq 
	p\left(
\left| \,\,
	\|\mbdelta_j\| - \sqrt{d}
\right| 
			> 		\epsilon
\right) 
	\leq 2 
	\exp(\nicefrac{-c\epsilon^2}{G^4}).
\]
Then, the following lower bound holds with probability at least $1- 2 
	\exp(\nicefrac{-c\epsilon^2}{G^4})$
\begin{align*}
\forall j \in V, \qquad \quad
        \left\langle \mbdelta_j, 
        \sum_{i\in V} \mbdelta_i\right\rangle & = \|\mbdelta_j\|^2
        \\
& \geq
(\sqrt{d} - \epsilon)^2 
        \\
&\geq  d - 2\epsilon\sqrt{d} 
        \\
&
\geq  d - 3\epsilon\sqrt{T_V d},
\end{align*}

To summarize this case, with the fact that  
$1- 2  \exp(\nicefrac{-c\epsilon^2}{G^4}) \geq 1 - 4 T_V \exp(\nicefrac{-c\epsilon^2}{G^4})$,
we have that 
\begin{align*}
\forall j \in V, \qquad \quad
        \left\langle \mbdelta_j, 
        \sum_{i\in V} \mbdelta_i\right\rangle
\geq  d - 3\epsilon\sqrt{T_V d},
\end{align*}
with probability at least $ 1 - 4 T_V \exp(\nicefrac{-c\epsilon^2}{G^4})$.

\paragraph{Case $T_V\geq 2$.}
First note that,
\begin{align*}
\forall j \in V \qquad \quad \left\langle \mbdelta_j, \sum_{i\in V} \mbdelta_i\right\rangle  = \|\mbdelta_j\|^2 +  \left\langle \mbdelta_j, \sum_{i\in V, i\not=j} \mbdelta_i\right\rangle 
\end{align*}

For each of the $T_V$ different 
$\mbdelta_j$'s, using \cref{lem:norm-conc} bounds  the probability of the norm $\|\mbdelta_j\|$ being larger than $\sqrt{d} - \epsilon$:
\[
	p\left(
\|\mbdelta_j\| < \sqrt{d} - \epsilon
\right)  
\leq 
	p\left(
\left| \,\,
	\|\mbdelta_j\| - \sqrt{d}
\right| 
			> 		\epsilon
\right) 
	\leq 2 
	\exp(\nicefrac{-c\epsilon^2}{G^4}).
\]

In the case where $T_V\geq 2$, we express the inner product of a vector and a sum of vectors as follows
\[\left\langle \mbdelta_j, \sum_{i\in V, i\not=j} \mbdelta_i\right\rangle  = \sqrt{T_V-1}\left\langle \mbdelta_j, \frac{1}{\sqrt{T_V-1}}\sum_{i\in V, i\not=j} \mbdelta_i\right\rangle ,\]
and noting that like above, $\frac{1}{\sqrt{T_V-1}}\sum_{i\in V, i\not=j} \mbdelta_i$ is a vector of standard normal random variables, 
we apply \cref{lem:inner-products} to get

\[
	\forall i\in V \qquad \quad 
p\left(
	\left|
		\left\langle \mbdelta_j, \sum_{i\in V, i\not=j} \mbdelta_i\right\rangle	\right|
		\geq \epsilon \sqrt{(T_V-1)d}
\right)
		\leq 2\exp\left(-c\frac{\epsilon^2}{G^4}\right).
		\]
		
Putting these together, by union bound over $V$
\begin{align*}
	p&\left[\forall j  \in V \qquad 
\Bigg(	\|\mbdelta_j\| < \sqrt{d} - \epsilon \Bigg) \,\,\text{ or } \,\,
	\left(
	\left|
		\left\langle \mbdelta_j, \sum_{i\in V, i\not=j} \mbdelta_i\right\rangle	\right|
		\geq \epsilon \sqrt{(T_V-1)d}
\right)
\right]
\\
	& \leq \sum_{j\in V}
	p
\Bigg(	\|\mbdelta_j\| < \sqrt{d} - \epsilon \Bigg)
+ 
	p\left(
	\left|
		\left\langle \mbdelta_j, \sum_{i\in V, i\not=j} \mbdelta_i\right\rangle	\right|
		\geq \epsilon \sqrt{(T_V-1)d}
\right)
\\
	& \leq 
\sum_{j\in V} 2\exp\left(-c\frac{\epsilon^2}{G^4}\right) + 2\exp\left(-c\frac{\epsilon^2}{G^4}\right)
\\
	& \leq 4T_V \exp\left(-c\frac{\epsilon^2}{G^4}\right).
\end{align*}

Thus, with probability at least $1 - 4T_V \exp\left(-c\frac{\epsilon^2}{G^4}\right),$ none of the events happen and
\begin{align*}
\forall j \in V  \qquad
	\left\langle \mbdelta_j, \sum_{i\in V} \mbdelta_i\right\rangle  
	 \quad & = \quad  \|\mbdelta_j\|^2 +  \left\langle \mbdelta_j, \sum_{i\in V, i\not=j} \mbdelta_i\right\rangle 
\\
	& \geq (\sqrt{d}-\epsilon)^2 - \epsilon\sqrt{(T_V-1)d}
\\
	& 
	= d - 2\epsilon\sqrt{d} + \epsilon^2 - \epsilon\sqrt{(T_V-1)d}
\\
	 &\geq d - 2\epsilon\sqrt{(T_V-1)d} - \epsilon\sqrt{(T_V-1)d} \qquad \quad \quad \quad
\\
	 &\geq d - 3\epsilon\sqrt{T_V d}
\end{align*}

Thus, putting the analysis in the two cases together, as long as $T_V \geq 1$
\[\forall j \in V  \qquad
	\left\langle \mbdelta_j, \sum_{i\in V} \mbdelta_i\right\rangle  \geq d - 3\epsilon\sqrt{T_V d}, \]
 with probability at least $1 - 4T_V \exp\left(-c\frac{\epsilon^2}{G^4}\right).$

\paragraph{Proof of \cref{eq:shortcut-inner-product}}
Next, we apply \cref{lem:inner-products} again to the inner product of two vectors of i.i.d standard normal random variables:

\begin{align*}
\forall j\in  U \qquad p\left(\quad \left|\left\langle \mbdelta_j , \frac{1}{\sqrt{T_V}}\sum_{i\in V} \mbdelta_i\right\rangle\right| \geq \epsilon \sqrt{{d}}\right) < 2\exp(\nicefrac{-c\epsilon^2}{G^4}).
\end{align*}

By union bound over $U$
\begin{align*}
\qquad p\left[\forall j\in  U  \qquad \left(\quad \left|\left\langle \mbdelta_j , \frac{1}{\sqrt{T_V}}\sum_{i\in V} \mbdelta_i\right\rangle\right| \geq \epsilon \sqrt{{d}}\right)\right] < 2 T_u \exp(\nicefrac{-c\epsilon^2}{G^4}).
\end{align*}

Thus, with probability at least $1 - 2 T_u \exp\left(-c\frac{\epsilon^2}{G^4}\right),$ the following holds, concluding the proof

\[
	\forall j\in  U  \qquad \quad \quad \left|\left\langle \mbdelta_j , \frac{1}{\sqrt{T_V}}\sum_{i\in V} \mbdelta_i\right\rangle\right| \leq \epsilon \sqrt{{d}}.
\]
\end{proof}

\begin{lemma}\label{lemma:dual}
Let $\{\mbx_i, \mby_i\}_{i\leq n}$ be a collection of $d$ dimensional covariates  $\mbx_i$ and label $\mby_i$ sampled according to $p_\rho$ in \cref{eq:sim-example}.
The covariates $\mbx_i=[\pm B \mby_i, \mby_i \mbdelta_i]$, where $+B$ in the middle coordinate for $i\in S_{\text{shortcut}}$ and $-B$ for $i\in S_{\text{leftover}}$. 
The dual formulation of the following norm-minimization problem
\begin{align*}
    \mbw_{\text{stable}} =  \argmin_{\mbw} & \quad  \mbw_y^2 + \mbw_z^2 + \|\mbw_e\|^2   \\
        \text{s.t. } & i\in S_{\text{shortcut}} \quad w_y + Bw_z + \mbw_e^\top \mby_i \mbdelta_i > 1 \\ 
    \text{s.t. } & i\in S_{\text{leftover}} \quad w_y - Bw_z + \mbw_e^\top \mby_i \mbdelta_i > 1 \\ 
        & \mbw_y \geq B \mbw_z
\end{align*}
is the following with $\zeta^\top= [-B, 1, \mathbf{0}^{d-2}]$,
\begin{align}
	  \max_{\lambda\geq 0, \nu \geq 0}- \frac{1}{4}\|\zeta \nu + X^\top\lambda\|^2  +    \ind^\top \lambda,
\end{align}
where $X$ is a matrix with $\mby_i\mbx_i$ as its rows.
\end{lemma}

\begin{proof}
	
We use Lagrange multipliers $\lambda \in \mathbb{R}^n, \nu \in \mathbb{R}$ to absorb the constraints and then use strong duality.
Letting $\zeta^\top= [-B, 1, \mathbf{0}^{d-2}]$, $X$ be a matrix where the $i$th row is $\mbx_{i}\mby_i$,
\begin{align}
\min_{\mbw} \quad & \|\mbw\|^2   \qquad \text{s.t.}  \qquad X \mbw - \ind \geq 0 \qquad \zeta^\top\mbw \geq 0  \nonumber
\\
& \text{has the same solution as } \nonumber
\\
\max_{\lambda \geq 0, \nu \geq 0} \min_{\mbw} \quad & \|\mbw \|^2 - (X \mbw - \ind)^\top \lambda - \nu \zeta^\top\mbw  \label{eq:full-eq}
\end{align}

Now, we solve the inner minimization to write the dual problem only in terms of $\lambda,\nu $. Solving the inner minimization involves solving a quadratic program, which is done by setting its gradient to zero, 
\begin{align*}
      & \nabla_{\mbw} \left(\|\mbw\|^2 - (X \mbw - \ind)^\top \lambda - \nu \zeta^\top\mbw \right) = 2\mbw - X^\top\lambda - \nu \zeta = 0
     \\
\implies & \mbw = \frac{1}{2}(\zeta \nu + X^\top \lambda)
\end{align*}

Substituting $\mbw =\frac{1}{2}(\zeta \nu + X^\top \lambda)$ in \cref{eq:full-eq}
\begin{align*}
\|\mbw \|^2 - & (X \mbw - \ind)^\top \lambda - \nu \zeta^\top\mbw 
=
\\
& \frac{1}{4}\|\zeta \nu + X^\top\lambda\|^2 - ( \frac{1}{2}(X(\zeta \nu + X^\top \lambda) - \ind)^\top \lambda -  \frac{1}{2}\nu \zeta^\top(\zeta \nu + X^\top \lambda)
\\
& = \frac{1}{4}\|\zeta \nu + X^\top\lambda\|^2 
    - 
(\frac{1}{2}(X(\zeta \nu + X^\top \lambda) - \ind)^\top \lambda 
    - 
         \frac{1}{2}\nu^2 \|\zeta\|^2
    - 
         \frac{1}{2}\nu \zeta^\top X^\top \lambda
\\
& =  
    \frac{1}{4}\|\zeta \nu + X^\top\lambda\|^2 
    - 
        \frac{1}{2} (X(X^\top \lambda))^\top \lambda 
    -
        \frac{1}{2} (X(\zeta \nu))^\top \lambda 
    + 
         \ind^\top \lambda 
    -
    \frac{1}{2} \nu^2 \|\zeta \|^2 -  \frac{1}{2}\nu \zeta^\top X^\top \lambda
\\
& =  
    \frac{1}{4}\|\zeta \nu + X^\top\lambda\|^2 
    -
    \left(
       \frac{1}{2}  (X(X^\top \lambda))^\top \lambda 
    +
        \frac{1}{2} \nu^2\|\zeta \|^2 
    + 
        \nu \zeta^\top X^\top \lambda
\right)
        + 
       \ind^\top \lambda 
\\
& =  
    \frac{1}{4}\|\zeta \nu + X^\top\lambda\|^2 
    -
    \left(
       \frac{1}{2}  (X^\top \lambda)^\top X^\top \lambda 
    +
        \frac{1}{2} \nu^2\|\zeta \|^2 
    + 
        \nu \zeta^\top X^\top \lambda
\right)
        + 
       \ind^\top \lambda 
\\
& =  
    \frac{1}{4}\|\zeta \nu + X^\top\lambda\|^2 
    -
      \frac{1}{2} \left(
      \|X^\top \lambda\|^2
    +
            \|\nu \zeta\|^2 
    + 
        2\nu \zeta^\top X^\top \lambda
\right)
        + 
       \ind^\top \lambda 
\\
& =  
    \frac{1}{4}\|\zeta \nu + X^\top\lambda\|^2 
    -
       \frac{1}{2}  \|\zeta \nu + X^\top\lambda\|^2 
    + 
          \ind^\top \lambda 
\\
& =  
    - \frac{1}{4}\|\zeta \nu + X^\top\lambda\|^2  +    \ind^\top \lambda 
\end{align*}
\end{proof}

\subsection{Shortcut learning in max-margin classification}
\setcounter{thm}{0}
\setcounter{corr}{0}

We repeat the \gls{dgp} from the linear perception task in \cref{eq:sim-example} here.
\begin{align}
\label{eq:sim-example-rebuttal}
    \mby \sim \textrm{Rad}, 
    \quad
    \mbz 
    \sim \begin{cases}
            p_\rho(\mbz = y \g \mby = y) = \rho \\
            p_\rho(\mbz = -y \g \mby = y) = (1 - \rho) \\
            \end{cases},
\quad \mbdelta  \sim \cN(0, \mathbf{I}^{d-2}),
    \quad 
    \mbx 
    = \left[B*\mbz,\mby,\mbdelta \right].
\end{align}

\newcommand{\scalingpropformal}{
Let $\mbw^*$ be the max-margin predictor
 on $n$ training samples from \cref{eq:sim-example-rebuttal} 
 with a leftover group of size $k$.
There exist constants $C_1, C_2, N_0 > 0$ such that
\begin{align}
      \forall & \,\,\, \text{integers} \,\,\, k\in\left(0,\frac{n}{10}\right) 
      \label{eq:k-upper-bound} 
\\
 \forall &\,\, d \geq C_1 k \log (3n), \label{eq:d-lower-bound}
	\\
\forall &\,\, B > C_2 \sqrt{\nicefrac{d}{k}}, \label{eq:B-lower-bound}
\end{align}
with probability at least $1-\frac{1}{3n}\,\,$ over draws of the training data, it holds that {$\,\,\, {B \mbw_z^*} > {\mbw_y^*}$}.}

\begin{thm}
\scalingpropformal{}
\end{thm}

\vspace{10pt}
Before giving the proof of \cref{thm:bad-case}, we first give the corollary showing overparameterization is not necessary for \cref{thm:bad-case} to hold.\\

\begin{corr}
For all $n > N_0$ --- where the constant $N_0$ is from \cref{thm:bad-case} --- with scalar $\tau\in (0,1)$ such that the dimension $d=\tau n < n$, \cref{thm:bad-case} holds. 
\[\forall k \leq n \times \min\left\{\frac{1}{10}, \frac{\tau}{C_1 \log 3n}\right\},\]
a linear model trained via  \derm{} yields a predictor $\mbw^*$ such that
 $ {B \mbw_z^*} > {\mbw_y^*}$.
 \end{corr}

\begin{proof}
We show that for a range of $k$, for all $n\geq N_0$ \cref{thm:bad-case} holds for some $d < n$.
Note that \cref{thm:bad-case} holds for $n\geq N_0, d = C_1 k \log (3n)$ and 
\[\forall k < \frac{n}{10}.\]

Setting $d \leq \tau n $ for some $\tau \in (0,1)$ such that $d < n$ means that \cref{thm:bad-case} holds if
\[C_1 k \log (3n) = d 
\leq \tau n \implies k\leq \frac{\tau n }{C_1\log (3n)}.\]

Absorbing this new upper bound into the requirements on $k$ for \cref{thm:bad-case} to hold, we get that for any scalar $n>N_0, \tau \in (0,1), d=\tau n$, \cref{thm:bad-case} holds for
\[\forall k < n\times \min \left\{\frac{1}{10}, \frac{\tau}{C_1 \log(3n)}  \right\}.\]
In turn, even though $d<n$, a linear model trained via \derm{} converges in direction to a max-margin classifier such that $\mbw^*$ with
 ${B \mbw_z^*} > {\mbw_y^*}$.
\end{proof}

\newcommand{\goodcaselemma}{Consider the following optimization problem from \cref{eq:good-case}  where $n$ samples of $\mbx_i, \mby_i$  come from \cref{eq:sim-example} where $\mbx_i \in \mathbf{R}^d$:
  \begin{align}
  \begin{split}
    \mbw_{\text{stable}} =  & \argmin_{\mbw} \quad  w_y^2 + w_z^2 + \|\mbw_e\|^2   \\
        \text{s.t. } & i\in S_{\text{shortcut}} \quad w_y + Bw_z + \mbw_e^\top \mby_i \mbdelta_i > 1 \\ 
    \text{s.t. } & i\in S_{\text{leftover}} \quad w_y - Bw_z + \mbw_e^\top \mby_i \mbdelta_i > 1 \\ 
        & \mbw_y \geq B \mbw_z
    \end{split}
  \end{align}
Let $k=|S_{\text{leftover}}| > 1$.
 Then,  for a fixed constant $c=\frac{1}{(2e)^2}$, with any scalar $\epsilon < \sqrt{d}$, with probability at least $1 - 2\exp(-\nicefrac{c\epsilon^2}{G^4})$ and $\forall \text{ integers }  M \in \left[1, \lfloor\frac{n}{2k}\rfloor\right]$,
 \[ \|\mbw_{\text{stable}}\|^2 \geq W_{\text{stable}} = \frac{1}{4 + \frac{\left(\sqrt{d} + \epsilon\right)^2}{2Mk}}.\]
 }

\newcommand{\badcaselemma}{	Consider the following optimization problem from \cref{eq:good-case} where $n$ samples of $\mbx_i, \mby_i$ come from \cref{eq:sim-example} where $\mbx_i \in \mathbf{R}^d$:
	  \begin{align}
  \begin{split}
    \mbw_{\text{shortcut}} = & \argmin_{\mbw} \quad  w_y^2 + w_z^2 + \|\mbw_e\|^2   \\
        \text{s.t. } & i\in S_{\text{shortcut}} \quad w_y + Bw_z + \mbw_e^\top \mby_i \mbdelta_i > 1 \\
         & i\in S_{\text{leftover}} \quad w_y - Bw_z + \mbw_e^\top \mby_i \mbdelta_i > 1 \\ 
        & \mbw_y < B \mbw_z
    \end{split}
  \end{align}
Let $k=|S_{\text{leftover}}| \geq 1 $.
 Then,  for a fixed constant $c=\frac{1}{(2e)^2}$, with any scalar $\epsilon <\frac{1}{3}\sqrt{\frac{d}{k}} < \sqrt{d}$, with probability at least $1 - 2(2k + (n-k) + 1)
\exp(-c\frac{\epsilon^2}{G^4})$, for  $\gamma = \frac{2}{d - 4\epsilon\sqrt{kd}}$,
\[ \|\mbw_{\text{shortcut}}\|^2 \leq W_{\text{shortcut}}= \gamma^2 k (\sqrt{d} + \epsilon)^2 +  \frac{\left(1 + \gamma\epsilon\sqrt{dk}\right)^2}{B^2}\]
}

\vspace{10pt}
\begin{proof} (of \cref{thm:bad-case})
 We consider two norm-minimization problems over $\mbw$, one under constraint $\mbw_y \geq B \mbw_z$ and another under $\mbw_y < B \mbw_z$.
We show that the latter achieves lower norm and therefore, max-margin will achieve solutions $\mbw_y < B \mbw_z$.
The two minimization problems are as follows:

\noindent\begin{minipage}{0.49\linewidth}
  \begin{align}
  \begin{split}
    \mbw_{\text{stable}} =  & \argmin_{\mbw} \quad  w_y^2 + w_z^2 + \|\mbw_e\|^2   \\
        \text{s.t. } & i\in S_{\text{shortcut}} \quad w_y + Bw_z + \mbw_e^\top \mby_i \mbdelta_i > 1 \\ 
    \text{s.t. } & i\in S_{\text{leftover}} \quad w_y - Bw_z + \mbw_e^\top \mby_i \mbdelta_i > 1 \\ 
        & \mbw_y \geq B \mbw_z
    \end{split}
    \label{eq:good-case}
  \end{align}
\end{minipage}
\hspace{10pt}
\begin{minipage}{0.48\linewidth}
  \begin{align}
  \begin{split}
    \mbw_{\text{shortcut}} = & \argmin_{\mbw} \quad  w_y^2 + w_z^2 + \|\mbw_e\|^2   \\
        \text{s.t. } & i\in S_{\text{shortcut}} \quad w_y + Bw_z + \mbw_e^\top \mby_i \mbdelta_i > 1 \\
         & i\in S_{\text{leftover}} \quad w_y - Bw_z + \mbw_e^\top \mby_i \mbdelta_i > 1 \\ 
        & \mbw_y < B \mbw_z
    \end{split}
    \label{eq:bad-case}
  \end{align}
  
\end{minipage}

From \cref{eq:good-case}, any $\mbw$ that satisfy the constraints of the dual maximization problem will lower bound the value of the optimum of the primal, $\|\mbw_{\text{stable}}\|^2 \geq W_{\text{stable}}$.
From the \cref{eq:bad-case}, substituting a guess in $\mbw_{\text{shortcut}}$ that satisfies the constraints yields an upper bound, $\|\mbw_{\text{shortcut}}\|^2 \leq W_{\text{shortcut}}$.
The actual computation of the bounds $W_{\text{shortcut}}, W_{\text{stable}}$ is in \cref{lemma:bad-case,lemma:good-case} which are proved in \cref{appsec:good-case-lemma} and \cref{appsec:bad-case-lemma} respectively.
We reproduce the lemmas here for convenience.\\

\begin{lemma*}$\boldsymbol{(7)}$ 
\goodcaselemma{}	
\end{lemma*}
\begin{lemma*}$\boldsymbol{(8)}$ 
\badcaselemma{}
\end{lemma*}

Together, the lemmas say that for any $\forall \text{ integers }  M \in \left[1, \lfloor\frac{n}{2k}\rfloor\right]$ and $\epsilon <\frac{1}{3}\sqrt{\frac{d}{k}}$, with probability $\geq 1 -2 \exp(\nicefrac{-c\epsilon^2}{G^4})$
 \[ \|\mbw_{\text{stable}}\|^2 \geq W_{\text{stable}} =  \frac{1}{4 + \frac{\left(\sqrt{d} + \epsilon\right)^2}{2Mk}}.\]
and with probability at least $1 - 2(2k + (n-k) + 1)
\exp(-c\frac{\epsilon^2}{G^4})$, for  $\gamma = \frac{2}{d - 4\epsilon\sqrt{kd}} > 0$,
\[ \|\mbw_{\text{shortcut}}\|^2 \leq W_{\text{shortcut}} =  \gamma^2 k (\sqrt{d} + \epsilon)^2 +  \frac{\left(1 + \gamma\epsilon\sqrt{dk}\right)^2}{B^2}\]
 
First, we choose $\epsilon^2 = 2\frac{G^4}{c}\log(3n)$.
This gives us the probability with which these bounds hold: as $k<0.1n$ we have $k+2<\frac{n}{2}$ and
\begin{align*}
	1 - 2(2k + (n-k) + 2)
\exp(-c\frac{\epsilon^2}{G^4}) & = 1 - 2(n + k + 2)\exp(-2\log(3n))
\\
&
	\geq 	1 - 2(\frac{3n}{2})\exp(-2\log(3n))
\\
&
	= 	1 - \exp(-2\log(3n)  + \log(3n))
\\
& =	1 - \exp(-\log(3n))
\\
&	= 1 - \frac{1}{3n}.
\end{align*}

Next, we will instantiate the parameter $M$ and set the constants $C_1,C_2$ and the upper bound on $k$ in \cref{thm:bad-case} to guarantee the following \cref{eq:separation}:
\begin{align}
 \tag{{separation inequality}}
\label{eq:separation}
W_{\text{shortcut}} = \gamma^2 k (\sqrt{d} + \epsilon)^2 +  \frac{\left(1 + \gamma\epsilon\sqrt{dk}\right)^2}{B^2} \quad <  \quad\frac{1}{4 + \frac{\left(\sqrt{d} + \epsilon\right)^2}{2Mk}} = W_{\text{stable}},
\end{align}
which then implies that $\|\mbw_{\text{shortcut}}\|^2 < \|\mbw_{\text{stable}}\|^2$, concluding the proof.

\paragraph{Invoking the conditions in \cref{thm:bad-case} and setting the upper bound on $k$.}
We will keep the $\epsilon$ as is for simplicity of reading but invoke the inequalities satisfied by $\log(3n)$ from \cref{thm:bad-case}: 
\[\exists \text{ constant } C_1, \qquad  d \geq C_1 k \log (3n).\]
Now we let $C_1 = 2\frac{G^4}{c C^2}$ for a constant $C \in \left(0, \frac{1}{3}\right)$\footnote{The $\frac{1}{3}$ comes from requiring that $\epsilon < \frac{1}{3}\sqrt{\frac{d}{k}}$ from \cref{lemma:bad-case}. \vspace{-15pt}}, such that
\begin{align}\label{eq:eps-upper-bound}
	\epsilon^2  = 2\frac{G^4}{c}\log(3n) < C^2 \frac{d}{k} \implies   
\epsilon < C \sqrt{\frac{d}{k}} 
\text{ and } 
	\epsilon \sqrt{kd} < C d.
\end{align}
We next find a $C \in \left(0, \frac{1}{3}\right)$ such that \cref{eq:separation} holds with $M=5$, which upper bounds $k$:
\[M < \frac{n}{2k} \implies \frac{k}{n} < \frac{1}{2M} = \frac{1}{10} \implies k < \frac{n}{10}.\]

\paragraph{Simplifying $W_{\text{shortcut}}$ and $W_{\text{stable}}$.}
To actually show $W_{\text{shortcut}}<W_{\text{stable}}$ in \cref{eq:separation}, we compare a simplified strict upper bound on the LHS $W_{\text{shortcut}}$ and a simplified strict lower bound on the RHS $W_{\text{stable}}$

For the simplification of the RHS $W_{\text{stable}}$ of \cref{eq:separation}, we will use the fact that $d \geq 2\frac{G^4}{c C^2}\log(3n) k $.
Given the assumption $n > N_0$, choosing $N_0$ to be an integer such that $\log(3N_0) \geq  \frac{40 c C^2}{G^4}$ means that $\log(3n) > \frac{40 c C^2}{G^4}$ and we have
\begin{align}\label{eq:upper-bound-on-4}
	\frac{d}{k} > 80 \implies \frac{d}{10 k}  > 8 \implies \frac{1}{2}\frac{d}{10 k} > 4
\end{align}
which gives us, for $M=5$,
\begin{align}
	W_{\text{stable}} & = \frac{1}{4 + \frac{\left(\sqrt{d} + \epsilon\right)^2}{2Mk}} 
\\
& = 
\frac{1}{4 + \frac{\left(\sqrt{d} + \epsilon\right)^2}{10k}}  
\\
	& \geq  \quad\frac{1}{\frac{3}{2}\frac{\left(\sqrt{d} + \epsilon\right)^2}{10k}} 
	 \qquad \qquad \{4 < \frac{1}{2}\frac{d}{10k} < \frac{1}{2}\frac{(\sqrt{d} + \epsilon)^2}{10k} \text{ from \cref{eq:upper-bound-on-4} } \}
\\
	& = \frac{20k}{3(\sqrt{d} + \epsilon)^2} 
\\
	& \geq \frac{20k}{3(\sqrt{d} + C\frac{\sqrt{d}}{\sqrt{k}})^2} 
	\qquad \qquad 
\{\epsilon < \frac{C\sqrt{d}}{\sqrt{k}} \text{ from \cref{eq:eps-upper-bound} } \}
\\
	& = \frac{20k}{3(1+\frac{C}{\sqrt{k}})^2d} 
\\
	& > \frac{20k}{3(1+C)^2d} \qquad \qquad 
    \{  k \geq 1 \}
\end{align}

Now, we produce a simpler upper bound on the first part of the LHS of \cref{eq:separation}:
recalling that $\gamma = \frac{2}{d - 4\epsilon\sqrt{kd}}$, and substituting in the upper bounds on $\epsilon$,
\begin{align}
\gamma^2 k (\sqrt{d} + \epsilon)^2  & =  \left( \frac{2(\sqrt{d} + \epsilon)}{d - 4\epsilon\sqrt{kd}}\right)^2 k \nonumber
\\
	& 
<
4  \left( \frac{(\sqrt{d} + C\sqrt{\frac{d}{k}})}{d - 4C d}\right)^2 k \nonumber
\qquad \qquad 
\{\epsilon < \frac{C\sqrt{d}}{\sqrt{k}} \text{ from \cref{eq:eps-upper-bound} } \}
\\
&	= \frac{4k}{d} \left( \frac{(1 + \frac{C}{\sqrt{k}})}{1 - 4C}\right)^2 \nonumber
\\
&	\leq \frac{4k}{d} \left( \frac{1 + C}{1 - 4C}\right)^2 \label{eq:up-lhs-first}, 
\qquad \qquad \{ k \geq 1\}
\end{align}

Next is a simpler upper bound on the second part of the LHS of \cref{eq:separation}.
Again with $\gamma = \frac{2}{d - 4\epsilon\sqrt{kd}}$, 
\begin{align*}
	\frac{\left(1 + \gamma\epsilon\sqrt{dk}\right)^2}{B^2}&  =   \frac{\left(1 + \frac{2\epsilon\sqrt{dk}}{{d - 4\epsilon\sqrt{kd}}}\right)^2}{B^2}
\\
	& \leq  \frac{\left(1 + \frac{2C d}{{d - 4Cd}}\right)^2}{B^2}
\\
	& = \frac{\left(1 + \frac{2C}{1 - 4C}\right)^2}{B^2}
\end{align*}

Now setting
\[B > \sqrt{2} \frac{{\left(1 + \frac{2C}{1 - 4C}\right)}{}}{\sqrt{\frac{4k}{d}}\left( \frac{1 + C}{1 - 4C}\right)}\]
gives the lower bound on $B$ from \cref{thm:bad-case}:
\[B > C_2 \sqrt{\frac{d}{k}}, \qquad \text{ where } \quad 
C_2 = \frac{{\left(1 + \frac{2C}{1 - 4C}\right)}{}}{\sqrt{2}\left( \frac{1 + C}{1 - 4C}\right)} = \frac{(1 - 2C)}{\sqrt{2}(1 + C)}.\]
Formally, 
\begin{align}\label{eq:upper-bound-b}
	B > C_2 \sqrt{\frac{d}{k}} \implies   \frac{\left(1 + \frac{2C}{1 - 4C}\right)^2}{B^2} <  \frac{1}{2}\left(\sqrt{\frac{4k}{d}}\left( \frac{1 + C}{1 - 4C}\right)\right)^2 
= \frac{1}{2} \frac{4k}{d} \left( \frac{1 + C}{1 - 4C}\right)^2.
\end{align}
By combining the upper bound  from \cref{eq:upper-bound-b} and the upper bound from \cref{eq:up-lhs-first},
we get an upper bound on the whole of the LHS of \cref{eq:separation}, which in turn provides an upper bound on $W_{\text{shortcut}}$:
\[ W_{\text{shortcut}} = \gamma^2 k (\sqrt{d} + \epsilon)^2 +  \frac{\left(1 + \gamma\epsilon\sqrt{dk}\right)^2}{B^2} \quad  < \frac{3}{2} \frac{4k}{d} \left( \frac{(1 + {C}{})}{1 - 4C}\right)^2 \leq \frac{3}{2} \frac{4k}{d} \left( \frac{(1 + C)}{1 - 4C}\right)^2,\]
because $k\geq 1$.
Note the upper bound is strict.

\paragraph{Concluding the proof.}

Now, we show that a $C$ exists such that the following holds, which implies $W_{\text{shortcut}} < W_{\text{stable}}$, which in turn implies \cref{eq:separation} and the proof concludes:
\[ W_{\text{shortcut}} < \frac{3}{2} \frac{4k}{d} \left( \frac{(1 + C)}{1 - 4C}\right)^2 \leq   \frac{20k}{3(1+C)^2d} < W_{\text{stable}}. \]

The above inequality holds when
\[ 6 \left( \frac{(1 + C)}{1 - 4C}\right)^2 \leq \frac{20}{3(1+C)^2} \quad \Longleftrightarrow \quad  
\left({1 + C}\right)^2 -  \sqrt{\frac{10}{9}}({1 - 4C}) \leq 0.\]

The right hand side holds when the quadratic equation $\left({1 + C}\right)^2 -  \sqrt{\frac{10}{9}}({1 - 4C})$ is non-positive, which holds between the roots of the equation.
The equation's positive solution is 
\[C = \frac{-3 + \sqrt{10}}{{3 + 2\sqrt{10} + \sqrt{5(8 + 3\sqrt{10})}}} \approx 0.008.\]
Setting $C$ to this quantity satisfies the requirement that $C\in(0,\frac{1}{3})$.\\

Thus, a $C$ exists such that \cref{eq:separation} holds which concludes the proof of \cref{thm:bad-case} for the following constants and constraints implied by $C$ and $M=5$:
\[C_2 = \frac{(1 - 2 C)}{\sqrt{2}(1 + C)} \qquad \qquad  C_1 = 2\frac{G^4}{c C^2} \qquad \qquad k < \frac{n}{10},\]
where $G$ is the $\psi_2$-norm of a standard normal r.v. and $c$ is the absolute constant from the Bernstein inequality in \cref{lem:Bernstein}.
\end{proof}

\subsection{Lower bounding the norm of solutions that rely more on the stable feature}\label{appsec:good-case-lemma}

\begin{lemma}\label{lemma:good-case}
\goodcaselemma{}
\end{lemma}

\begin{proof}
By \cref{lemma:dual}, the dual of \cref{eq:good-case} is the following for $\zeta=[-B, 1, \mathbf{0}^{d-2}]$ and $X$ is an $ n\times d $ matrix with rows $\mby_i\mbx_i$:
\begin{align}
\label{eq:dual}
	  \max_{\lambda\geq 0, \nu \geq 0}- \frac{1}{4}\|\zeta \nu + X^\top\lambda\|^2  +    \ind^\top \lambda 
\end{align}
Now by duality
\[\|\mbw_{\text{stable}}\|^2  \geq 	  \max_{\lambda\geq 0, \nu \geq 0}- \frac{1}{4}\|\zeta \nu + X^\top\lambda\|^2  +    \ind^\top \lambda,\]
which means any feasible candidate to \cref{eq:dual} gives a lower bound on $\|\mbw_{\text{stable}}\|^2$.

\paragraph{Feasible Candidates for $\lambda, \nu$.}
We now define a set $U\subset [n]$, and let $\lambda_i=\frac{\alpha}{|U|} > 0$ for $i\in U$ and $0$ otherwise. 
For $M \in (1, \lfloor\frac{n}{2k}\rfloor]$, we take $2Mk$ samples from the training data to be included in $U$. Formally, 
\[U = S_{\text{leftover}} \cup (2M-1)k \text{ a random samples from } S_{\text{shortcut}},\]
which gives the size $|U|=2Mk$. Then, we let ${\nu=\alpha\frac{2(M-1)}{2M}}>0$.\\

Note that for the above choice of $\lambda$, $X^\top \lambda$  is a sum of the rows from $U$ scaled by $\frac{\alpha}{|U|}$.
Adding up $k$ rows from $S_{\text{leftover}}$ and $k$ rows from $S_{\text{shortcut}}$ cancels out the $B$s and, so in the $B$ is accumulated $|U| - 2k = 2(M-1)k$ times, and so
\[X^\top \lambda = \left[ \alpha * \frac{|U|-2k}{|U|} B ,  \alpha, \frac{\alpha}{|U|} \sum_{i\in U}\mbdelta_{i} \right]= \left[\alpha B\frac{2(M-1)}{2M},  \alpha, \frac{\alpha}{|U|} \sum_{i\in U}\mbdelta_{i}\right].\]

As $\lambda$ has $\frac{\alpha}{|U|}$ on $|U|$ elements and $0$ otherwise, $\lambda^\top \ind  = {\alpha}$

As we set
$\nu=\alpha\frac{2(M-1)}{2M}$,
\begin{align}
\nu \zeta + X^\top \lambda & = \left[- \alpha B \frac{2(M-1)}{2M} + \alpha\frac{2(M-1)}{2M} B, \alpha\frac{2(M-1)}{2M} + \alpha , 0 +  \frac{\alpha}{|U|}\sum_{i}\mbdelta_i\right] 
\\
& = 
\left[0\quad ,\alpha \left(1+\frac{2(M-1)}{2M}\right),\quad \frac{\alpha}{|U|}\sum_{i}\mbdelta_i\right]
\\
\implies 
   \|\zeta \nu + X^\top\lambda\|^2  
    & = 
\left\| \left[ 0,\alpha\left(1+\frac{2(M-1)}{2M}\right),  \frac{\alpha}{|U|} \sum_{i\in U}\mbdelta_{i}\right] \right\|^2
\\
    & = \alpha^2\left\| \left[0,\left(1+\frac{2(M-1)}{2M}\right),  \frac{1}{|U|} \sum_{i\in U}\mbdelta_{i}\right]\right\|
    \label{eq:lb-unopt}
\end{align}

For the chosen values of $\nu, \lambda$ the value of the objective in \cref{eq:dual} is
\begin{align}
    \frac{-\alpha^2}{4} \left\| \left[0,   \left(1+\frac{2(M-1)}{2M}\right),  \frac{1}{|U|} \sum_{i\in U}\mbdelta_{i}\right]\right\|^2 + \alpha
\end{align}

Letting \[\Gamma =  \left\| \left[0,  \left(1+\frac{2(M-1)}{2M}\right), \frac{1}{|U|} \sum_{i\in U}\mbdelta_{i}\right]\right\|^2 ,\] the objective is of the form $\alpha - \frac{\alpha^2 \Gamma}{4}$.
To maximize with respect to $\alpha$, setting the derivative of the objective w.r.t $\alpha$ to $0$ gives:
\[1 - \frac{2\alpha \Gamma}{4}  = 0 \implies \alpha = \frac{2}{\Gamma} \implies \alpha - \frac{\alpha^2 \Gamma}{4} = \frac{2}{\Gamma} - \frac{4}{\Gamma^2}\frac{\Gamma}{4} = \frac{1}{\Gamma}.\]
This immediately gives us
\[\|\mbw_{\text{stable}}\|^2 \geq \frac{1}{\Gamma},\]
and we lower bound this quantity by upper bounding $\Gamma$.

By concentration of gaussian norm as in \cref{lem:norm-conc}, with probability at least $1-2\exp(-c \frac{\epsilon^2}{G^4})$
\[\left\|\frac{1}{|U|}\sum_{i\in U}\mbdelta_i\right\| = \frac{1}{\sqrt{|U|}}\left\|\frac{1}{\sqrt{|U|}}\sum_{i\in U}\mbdelta_i\right\|\leq \frac{1}{\sqrt{|U|}}(\sqrt{d} + \epsilon).\]

In turn, recalling that $|U| = 2Mk$
\[\Gamma \leq \left(\frac{(2(M-1) + 2M)}{2M}\right)^2 + \left(\frac{\sqrt{d} + \epsilon}{\sqrt{|U|}}\right)^2
<  4 + \left(\frac{\sqrt{d} + \epsilon}{\sqrt{|U|}}\right)^2
\leq 4 + \frac{\left(\sqrt{d} + \epsilon\right)^2}{2Mk}
\]

The upper bound on $\Gamma$ gives the following lower bound  on $\|\mbw_{\text{stable}}\|^2$:
\[
\|\mbw_{\text{stable}}\|^2 \geq \frac{1}{\Gamma} \geq \frac{1}{4 + \frac{\left(\sqrt{d} + \epsilon\right)^2}{2Mk}}
\]
\end{proof}

\subsection{Upper bounding the norm of solutions that rely more on the shortcut.}\label{appsec:bad-case-lemma}

\begin{lemma}\label{lemma:bad-case}
\badcaselemma{}
\end{lemma}

\begin{proof}
Let $k=|S_{\text{leftover}}|$.
The candidate we will evaluate the objective for is
\begin{align}
\label{eq:primal-candidate}
	\mbw = \left[\frac{\beta}{B}, 0 , \gamma\sum_{j\in S_{\text{leftover}}} \mby_j \mbdelta_j\right].
\end{align}

\paragraph{High-probability bounds on the margin achieved by the candidate and norm of $\mbw$}

The margins on the shortcut group and the leftover group along with the constraints are as follows:
\begin{align}
\label{eq:margins}
	\begin{split}
		\forall j \in S_{\text{shortcut}} \quad m_j = 0 + B*\frac{\beta}{B} + \left \langle \mby_j \mbdelta_j, \gamma\sum_{i\in S_{\text{leftover}}}  \mby_i \mbdelta_i \right\rangle \geq 1
		\\
	\forall j \in S_{\text{leftover}} \quad m_j = 0 - B*\frac{\beta}{B} + \left \langle \mby_j \mbdelta_j, \gamma\sum_{i\in S_{\text{leftover}}} \mby_i \mbdelta_i \right\rangle \geq 1.
	\end{split}
\end{align}

Due to the standard normal distribution being isotropic,  and $\mby_j\in\{-1,1\}$, $\mby_j\mbdelta_j$ has the same distribution as $\mbdelta_j$.
Then, we apply \cref{lemma:concentration} with $V=S_{\text{leftover}}, U= S_{\text{shortcut}}$ --- which means $T_v=k$ and $T_u = (n-k)$ --- to bound the margin terms in \cref{eq:margins} and $\|\mbw\|^2$ with probability at least 
\[1 - 2(2k + (n-k) + 2)
\exp(-c\frac{\epsilon^2}{G^4}).\]

Applying the bound in \cref{eq:shortcut-inner-product} in \cref{lemma:concentration} between a sum of vectors and a different i.i.d vector,
\begin{align}
\label{eq:shortcut-margin}
	\forall j \in S_{\text{shortcut}} \qquad \left|\left \langle \mby_j \mbdelta_j, \gamma\sum_{i\in S_{\text{leftover}}} \mby_i \mbdelta_i \right\rangle\right|
\leq  
\gamma \epsilon\sqrt{k d} 
\end{align}
Applying the bound in \cref{eq:leftover-inner-product} from \cref{lemma:concentration}
\begin{align}\label{eq:leftover-margin}
	\forall j \in S_{\text{leftover}} \qquad \left \langle \mby_j \mbdelta_j, \gamma\sum_{i\in S_{\text{leftover}}}  \mby_i \mbdelta_i \right\rangle  
	& \geq 
			\gamma\left(d - 3\epsilon\sqrt{k d}\right)
\end{align}


The margin constraints on the shortcut and leftover from \cref{eq:margins} respectively imply
\[\beta - 		
		\gamma\epsilon\sqrt{dk}
	\geq 1 
\qquad \qquad 
-\beta +		
	\gamma 
		\left(
			d - 3\epsilon\sqrt{k d}
		\right)
 \geq 1 \]


We choose $\beta = 1 + \gamma\epsilon\sqrt{dk}$, which implies an inequality that $\gamma$ has to satisfy the following, which is due to $d - 3\epsilon\sqrt{kd} > 0$,
\[
- (1 + \gamma\epsilon\sqrt{dk}) + \gamma 
		\left(
			d - 3\epsilon\sqrt{k d}
		\right)
 \geq 1 \implies \gamma \geq \frac{2}{d - 4\epsilon\sqrt{kd}}
\]

Now, we choose
\[ \gamma = \frac{2}{d - 4\epsilon\sqrt{kd}}.\]

\paragraph{Computing the upper bound on the value of the objective in the primal problem in \cref{eq:bad-case}}
%
The feasible candidate's norm $\|\mbw\|^2$ is an upper bound on the solution's norm $\|\mbw_{\text{shortcut}}\|^2$ and so 
\[\|\mbw_{\text{shortcut}}\|^2 \leq \|\mbw\|^2 = \frac{1}{B^2}\beta^2 + \left\|\gamma \sum_{j\in S_{\text{leftover}}} \mby_j \mbdelta_j \right\|^2 = \gamma^2 k \left\|\frac{1}{\sqrt{k}}\sum_{j\in S_{\text{leftover}}} \mbdelta_j \right\|^2 + \frac{\beta^2}{B^2}\]

By \cref{lemma:concentration} which we invoked,
\[\left\|\frac{1}{\sqrt{k}}\sum_{j\in S_{\text{leftover}}} \mbdelta_j \right\|^2
 \leq (\sqrt{d} + \epsilon)^2.
 \]
 
To conclude the proof, substitute $\beta=1 + \gamma\epsilon\sqrt{dk}$ and get the following upper bound with $\gamma = \frac{2}{d - 3\epsilon\sqrt{kd}}$:
\[\|\mbw_{\text{shortcut}}\|^2 \leq \gamma^2 k (\sqrt{d} + \epsilon)^2 + \frac{\beta^2}{B^2} = \gamma^2 k (\sqrt{d} + \epsilon)^2 +  \frac{\left(1 + \gamma\epsilon\sqrt{dk}\right)^2}{B^2}.\]
\end{proof}

\subsection{Concentration of $k$ and intuition behind \cref{thm:bad-case}}\label{appsec:intuition}

\paragraph{Concentration of $k$ around $(1-\rho)n$.}

Denote the event that the $i$th sample lies in the leftover group as $I_i$: then $E[I_i] = 1-\rho$ and the leftover group size is $k=\sum_i I_i$.
Hoeffding's inequality (Theorem 2.2.6 in \citep{vershynin2018high}) shows that for any $t>0$, $k$ is at most $(1-\rho)n + t\sqrt{n}$ with probability at least $1 - \exp(-2t^2)$:
\[p\left(k - (1-\rho)n > t \sqrt{n}\right) = p\left(\sum_i \left( I_i -  (1-\rho)\right)  > t \sqrt{n} \right) = p\left(\sum_i \left( I_i -  E[I_i]\right)  > t \sqrt{n} \right) \leq \exp(-2t^2).\]

Letting $\rho = 0.9 + \sqrt{\frac{\log 3n}{n}}$ and $t=\sqrt{\log 3n }$, gives us
\begin{align*}
p\left(k - (1-\rho)n > t \sqrt{n}\right) 
& = 
p\left(k - 0.1n + \sqrt{n\log 3n } > \sqrt{\log 3n} \sqrt{n}\right)
\\
& = p\left(k - 0.1n > 0 \right)
\\
& \leq \exp(-2t^2)
\\
&  = \exp(-2\log 3n ).
\\
&  = \left(\frac{1}{3n}\right)^2
\\
&  < \frac{1}{3n}
\end{align*}

To connect $\rho$ to shortcut learning due to max-margin classification, we take a union bound of the event that $k< 0.1n$, which occurs with probability at least $1-\frac{1}{3n}$ and \cref{thm:bad-case} which occurs with probability at least $1- \frac{1}{3n}$.
This union bound guarantees that with probability at least $1 - \frac{2}{3n}$ over sampling the training data, max-margin classification on $n$ training samples from \cref{eq:sim-example} relies more on the shortcut feature if $\rho$ is above a threshold; and this threshold converges to $0.9$ at the rate of $\sqrt{\nicefrac{\log 3n }{n}}$.

\subsection{Bumpy losses improve ERM in  the under-parameterized setting}\label{appsec:bumpy-loss}

\begin{thm}
\bumpylossthm{}
\end{thm}

\begin{proof}
Letting $X$ be the matrix where each row is $\mby_i \mbx_i$, the theorem statement says the solution $\mbw^*$
\begin{align}
\label{eq:bumpy-minima}
    X\mbw^* = b \ind 
\end{align}

First, split $\mbw^* = [w_z^*, w_y^*, \mbw_{-{y}}^*]$. \Cref{eq:bumpy-minima} says that the margin of the model on any sample satisfies
\[\mby (\mbw^*)^\top \mbx = w_y^* \mby^2 +  w_z^*  \mby \mbz +  \mby ( \mbw_{-{y}}^*)^\top \mbdelta = b \qquad \implies \qquad  \mby (\mbw_{-{y}}^*)^\top \mbdelta = b- w_y^* \mby^2 -  w_z^*  \mby \mbz\]
We collect these equations for the whole training data by splitting $X$ into columns: denoting $Y, Z$ as vectors of $\mby_i$ and $\mbz_i$ and using $\cdot$ to denote element wise operation, split $X$ into columns that correspond to $\mby, \mbz$ and $\mbdelta$ respectively as $X = [ Y\cdot Y  \mid Y\cdot Z  \mid X_\delta ]$.
Rearranging terms gives us
\[w_z^* Y\cdot Z + w_y^*\ind  + X_\delta\mbw^*_\delta = b\ind  \qquad  \implies \qquad X_\delta\mbw^*_\delta = (b - w_y^*)\ind - w_z^* Y\cdot Z.\]
The elements of $ Y\cdot Z$ lie in $\{-1, 1\}$ and, as the shortcut feature does not always equal the label, the elements of $Y \cdot Z$  are not all the same sign.

\paragraph{Solutions do not exist when one non-zero element exists in $(b - w_y^*)\ind - w_z^* Y\cdot Z$}

By definition of $\mbw^*$
\[X_\delta\mbw^*_\delta = (b - w_y^*)\ind - w_z^* Y\cdot Z.\]

Denote $r =  (b - w_y^*)\ind- w_z^* Y\cdot Z.$ and  $A = X_\delta$.
Now we show that w.p. 1 solutions do not exist for the following system of linear equations:
\[Aw = r.\]
First, note that $A=X_\delta$ has $\mby_i\mbdelta_i$ for rows and as $\mby_i\indep \mbdelta_i$ and $\mby_i\in\{-1, 1\}$, each vector $\mby_i\mbdelta_i$ is distributed identically to a vector of independent standard Gaussian random variables.
Thus, $A$ is a matrix of IID standard Gaussian random variables.

Let $U$ denote $D-2$ indices such that the corresponding rows of $A$ form a matrix $D-2\times D-2$ matrix and $r_U$ has at least one non-zero element; let $A_U$ denote the resulting matrix.
Now $A_U$ is a ${D-2}\times {D-2}$ sized matrix
where each element is a standard Gaussian random variable.
Such matrices have rank $D-2$ with probability 1 because square singular matrices form a measure zero set under the Lebesgue measure over $\mathbf{R}^{D - 2\times D-2}$\citep{feng2007rank}.

We use subscript $\cdot_{-U}$ to denote all but the indices in $U$.
The equation $Aw = r$ implies the following two equations: 
\[A_Uw = r_U \qquad \qquad A_{-U}w = r_{-U}.\]
As $A_U$ is has full rank ($D-2$), $A_U w = r_U$ admits a unique solution $\mbw^*_U \not=0$ --- because $r_U$ has at least one non-zero element by construction.
Then, it must hold that
\begin{align}\label{eq:hyperplane}
	A_{-U}\mbw^*_U = r_{-U}.
\end{align}

For any row $v^\top \in A_{-U}$, \cref{eq:hyperplane} implies that $v^\top \mbw^*$ equals a fixed constant.
As $v$ is a vector of i.i.d standard normal random variables, $v^\top \mbw^*$ is a gaussian random variable with mean $\sum(\mbw^*_i)$ and variance $\|\mbw^*\|^2$.
Then with probability $1$, $v^\top \mbw^*$ will not equal a constant.
Thus, w.p.1 $A_{-U}\mbw^*_U = r_{-U}$ is not satisfied, which means w.p.1 there are no solutions to $A \mbw = r$.

\paragraph{Case where $(b - w_y^*)\ind - w_z^* Y\cdot Z$ is zero element-wise}

As $X$ has rank $D-2$, $X_\delta \mbw^*_\delta= 0$ only when $\mbw^*_\delta=0.$

Each element in $ (b - w_y^*)\ind- w_z^* Y\cdot Z$ is either $b - w_y^* + w_z^*$ or $b - w_y^* - w_z^*$.
Thus, 
\begin{align}
   (b - w_y^*)\ind - w_z^* Y\cdot Z =0  \quad \implies 
   \begin{cases}
 & b - w_y^* + w_z^* = 0, \\
& b- w_y^* - w_z^*=0
 \end{cases}
\end{align}
Adding and subtracting the two equations on the right gives
\[ 2(b - w_y^*) =0 \qquad \text{and} \qquad 2w_z^* = 0.\]
Thus, $\mbw^*_\delta=0, w^*_ z=0, b=w^*_y$.
\end{proof}

\section{Appendix: further experimental details and results}\label{sec:app}

\subsection{\Derm{} with $\ell_2$-regularization.}\label{appsec:l2-reg}

In \cref{sec:vdm_validate}, we show \derm{} achieves zero training loss by using the shortcut to classify the shortcut group and noise to classify the leftover group, meaning the leftover group is overfit.
The usual way to mitigate overfitting is via $\ell_2$-regularization, which, one can posit, may encourage models to rely on the perfect stable feature instead of the imperfect shortcut and noise.

We train the linear model from \cref{sec:vdm_validate} with \derm{} and $\ell_2$-regularization --- implemented as weight decay in the AdamW optimizer \citep{loshchilov2017decoupled} --- on data from \cref{eq:sim-example} with $d=800, B=10, n=1000$.
\Cref{fig:wd10} plots accuracy and losses for the $\ell_2$-regularized \derm{} with the penalty coefficient set to $10^{-8}$; it shows that $\ell_2$-regularization leads \derm{} to build models that only achieve $\approx 50\%$ test accuracy.

For smaller penalty coefficients, \derm{} performs similar to how it does without regularization, and for larger ones, the test accuracy gets worse than \derm{} without regularization.
We give an intuitive reason for why larger $\ell_2$ penalties may lead to larger reliance on the shortcut feature. 
Due to the scaling factor $B=10$ in the synthetic experiment, for a fixed norm budget, the model achieves lower loss when using the shortcut and noise compared to using the stable feature.
In turn, heavy $\ell_2$-regularization forces the model to rely more on the shortcut to avoid the cost of larger weight needed by the model to rely on the stable feature and the noise.

\begin{figure}[t]
    \centering
    \includegraphics[width=0.9\textwidth]{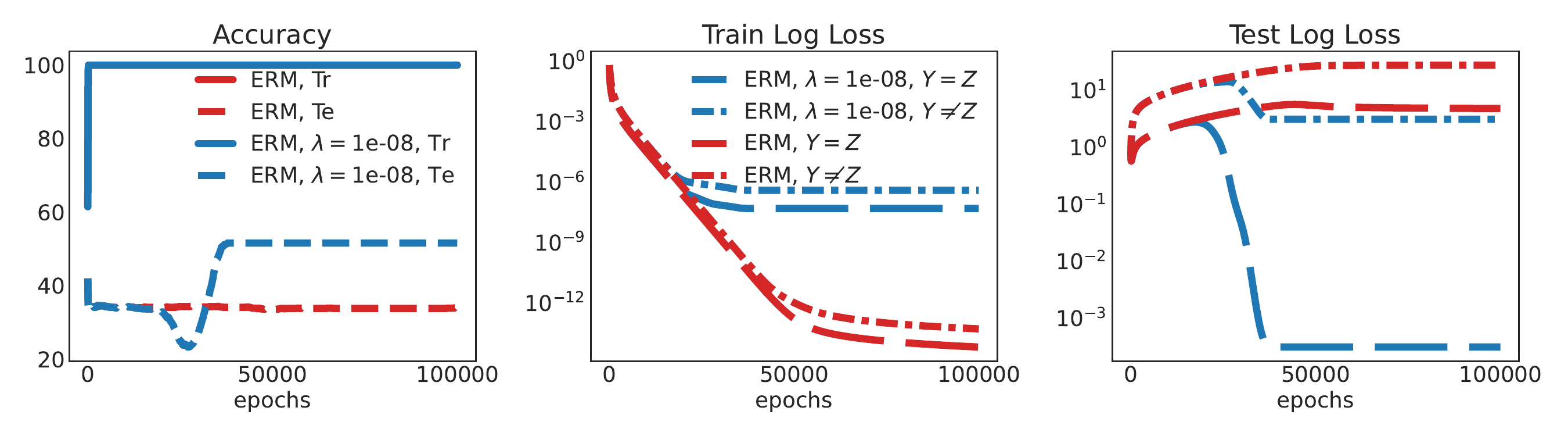}
    \caption{
\Derm{} with $\ell_2$-regularization with a penalty coefficient of $\lambda=10^{-8}$ achieves a test accuracy of $\approx 50\%$ , outperforming \derm{}.
The right panel shows that $\ell_2$-regularization leads to lower test loss on the minority group, meaning that the regularization does mitigate some overfitting.
However, the difference between the shortcut and leftover test losses shows that the model still relies on the shortcut.
    }
    \label{fig:wd10}
\end{figure}

\begin{figure}[t]
    \centering
    \includegraphics[width=0.7\textwidth]{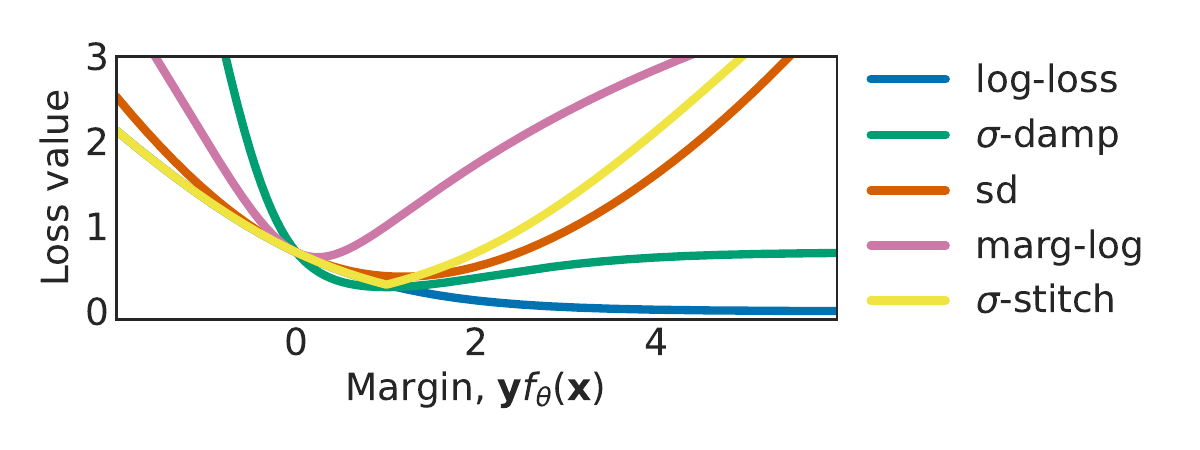}
    \caption{Comparing log-loss with \gls{cc} as functions of the margin. 
Each \gls{cc} loss has a "bump" which characterizes the loss function's transition from a decreasing function of the margin to an increasing one.
These bumps push models to have uniform margins because the loss function's derivative after the bump is negative which discourages large margins.
The hyperparameters (temperature in \sigdamp{} or function output target in \textsc{marg-log}.) affect the location of the bump and the slopes of the function on either side of the bump.
 }
    \label{fig:losses_plot}
\end{figure}

\subsection{\glsreset{cc}\Gls{cc}}\label{appsec:exps-synthetic-losses}

In \cref{fig:losses_plot}, we plot the different \gls{cc} losses along with log-loss.
Each \gls{cc} loss has a "bump" which characterizes the loss function's transition from a decreasing function of the margin to an increasing one.
These bumps push models to have uniform margins because the loss function's derivative after the bump is negative which discourages large margins.
The hyperparameters --- like temperature in \sigdamp{} or function output target in \textsc{marg-log} --- affect the location of the bump and the slopes of the function on either side of the bump.

\subsection{\Gls{cc} on a linear model}

In \cref{fig:lin-sigstitch}, we compare \derm{} to \stitch{}.
In \cref{fig:lin-sd} and \cref{fig:lin-sd-log}, compare \gls{sd} and \textsc{marg-log} respectively to \derm{}.
The left panel of all figures shows that \gls{cc} achieves better test accuracy than \derm{}, while the right most panel shows that the test loss is better on the leftover group using \gls{cc}. 
Finally, the middle panel shows the effect of controlling margins in training; namely, the margins on the training data do not go to $\infty$, evidenced by the training loss being bounded away from $0$.
Depending on the shortcut feature leads to different margins and therefore test losses between the shortcut and leftover groups; the right panel in each plot shows that the the test losses on both groups reach similar values, meaning \gls{cc} mitigates dependence on the shortcut.
While \derm{} fails to perform better than chance ($50\%$) even after $100,000$ epochs (see \cref{fig:lin-erm-failure}), \gls{cc} mitigates  shortcut learning within $5000$ epochs and achieves $100\%$ test accuracy.

\begin{figure}[t]
    \centering
    \includegraphics[width=0.9\textwidth]{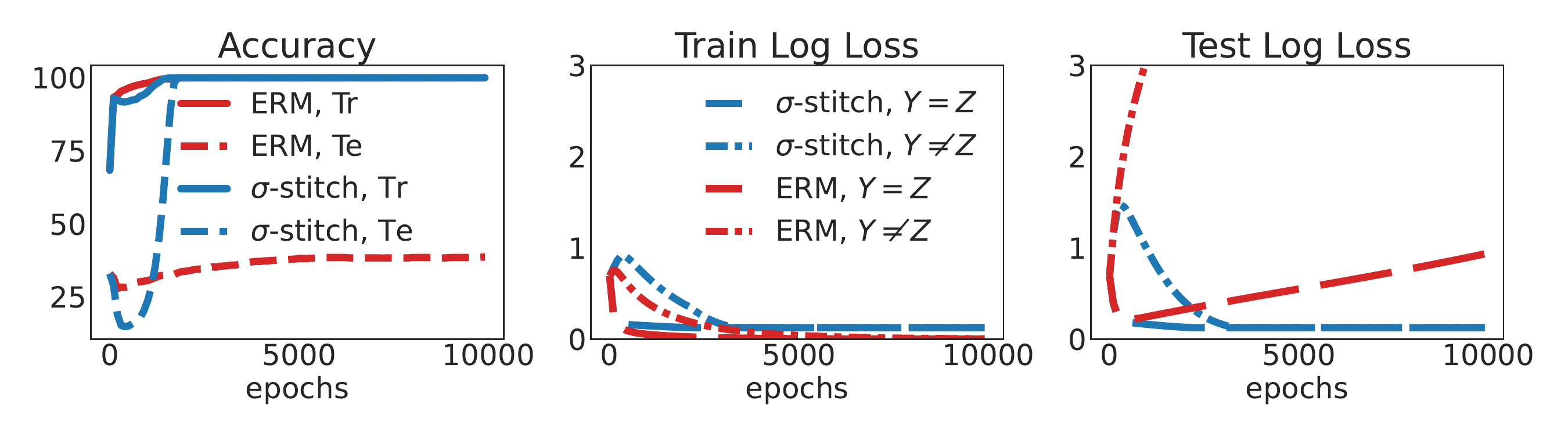}
    \caption{
    A linear trained with \stitch{} depend on the perfect stable feature to achieve perfect test accuracy, unlike \derm{}.
The middle panel shows that \stitch{} does not let the loss on the training shortcut group to go to zero, unlike \derm{}, and the right panel shows the test leftover group loss is better.
    }
    \label{fig:lin-sigstitch}
       \vspace{-10pt}
\end{figure}

\begin{figure}[t]
    \centering
    \includegraphics[width=0.9\textwidth]{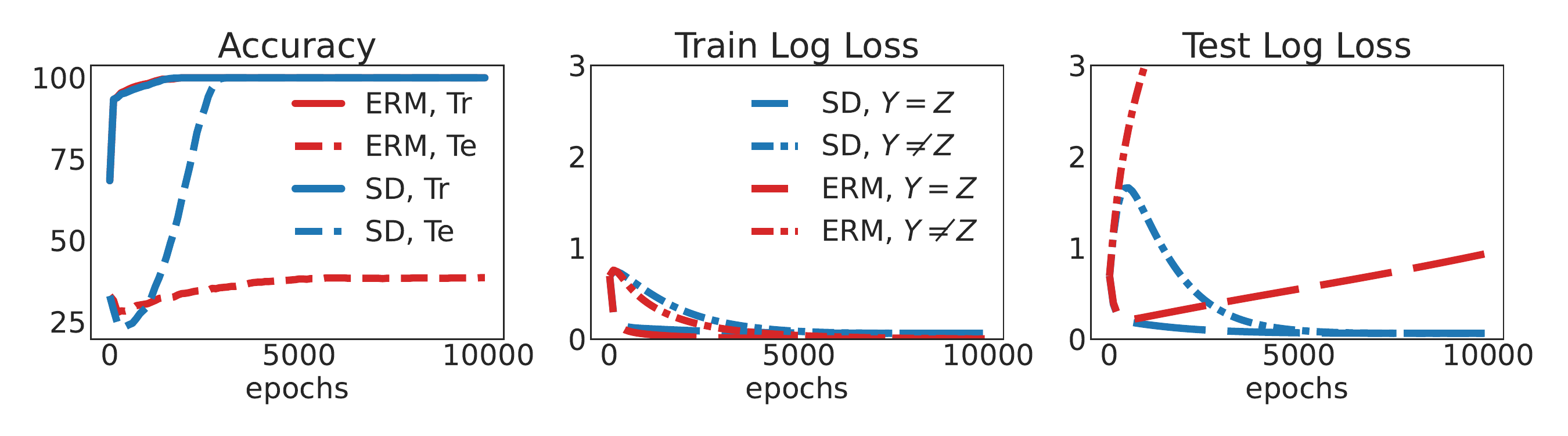}
    \caption{
A linear model trained with \gls{sd} depend on the perfect stable feature to achieve perfect test accuracy whereas \derm{} performs worse than random chance.
The middle panel shows that \gls{sd} does not let the loss on the training shortcut group to go to zero, unlike vanilla \derm{}, and the right panel shows the test-loss is better for the leftover group.
    }
    \label{fig:lin-sd}
\end{figure}

\begin{figure}[b]
\centering
    \includegraphics[width=0.9\textwidth]{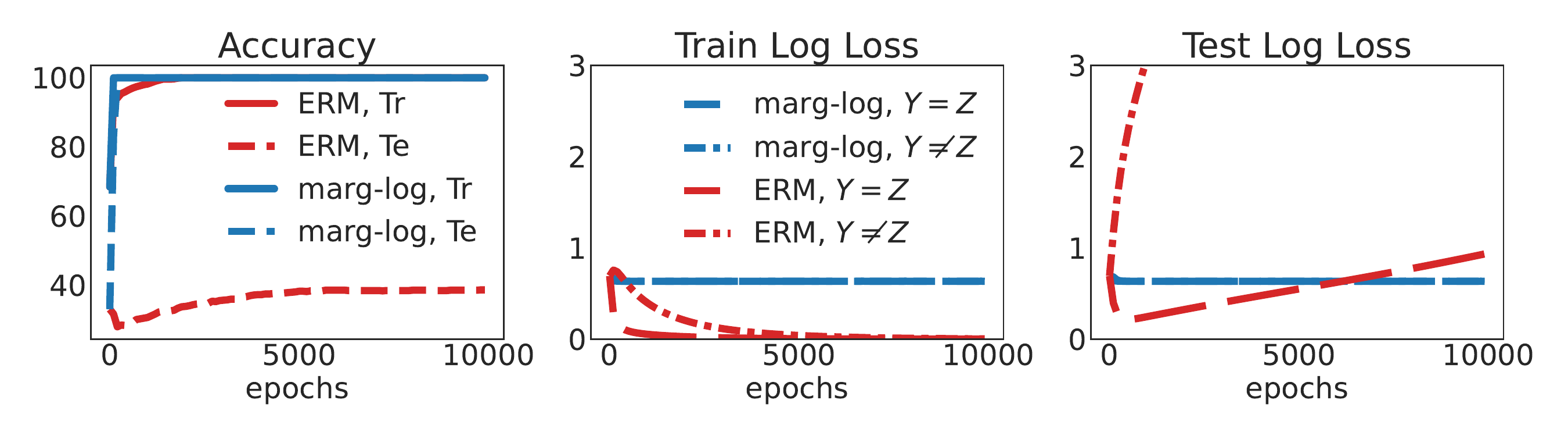}
    \caption{
A linear model trained with \textsc{marg-log} depend on the perfect stable feature to achieve perfect test accuracy whereas \derm{} performs worse than random chance.
The middle panel shows that \textsc{marg-log} does not let the loss on the training shortcut group to go to zero, unlike \derm{}, and the right panel shows the test-loss is better for the leftover group.
    }
    \label{fig:lin-sd-log}
\end{figure}

\subsection{\gls{cc} vs. \derm{} with a neural network}\label{cc-on-nn}

With $d=100$ and $B=10$ in \cref{eq:sim-example}, we train a two layer neural network on $3000$ samples from the training distribution.
The two layer neural network has a $200$ unit hidden layer that outputs a scalar.
\Cref{fig:nonlin-erm-failure} shows that a neural network trained via \derm{} fails to cross $50\%$ test accuracy even after $40,000$ epochs, while achieving less than $10^{-10}$ in training loss.

\begin{figure}[t]
\small
\centering
\begin{subfigure}[b]{0.6\textwidth}
\centering
\includegraphics[width=1.05\textwidth]{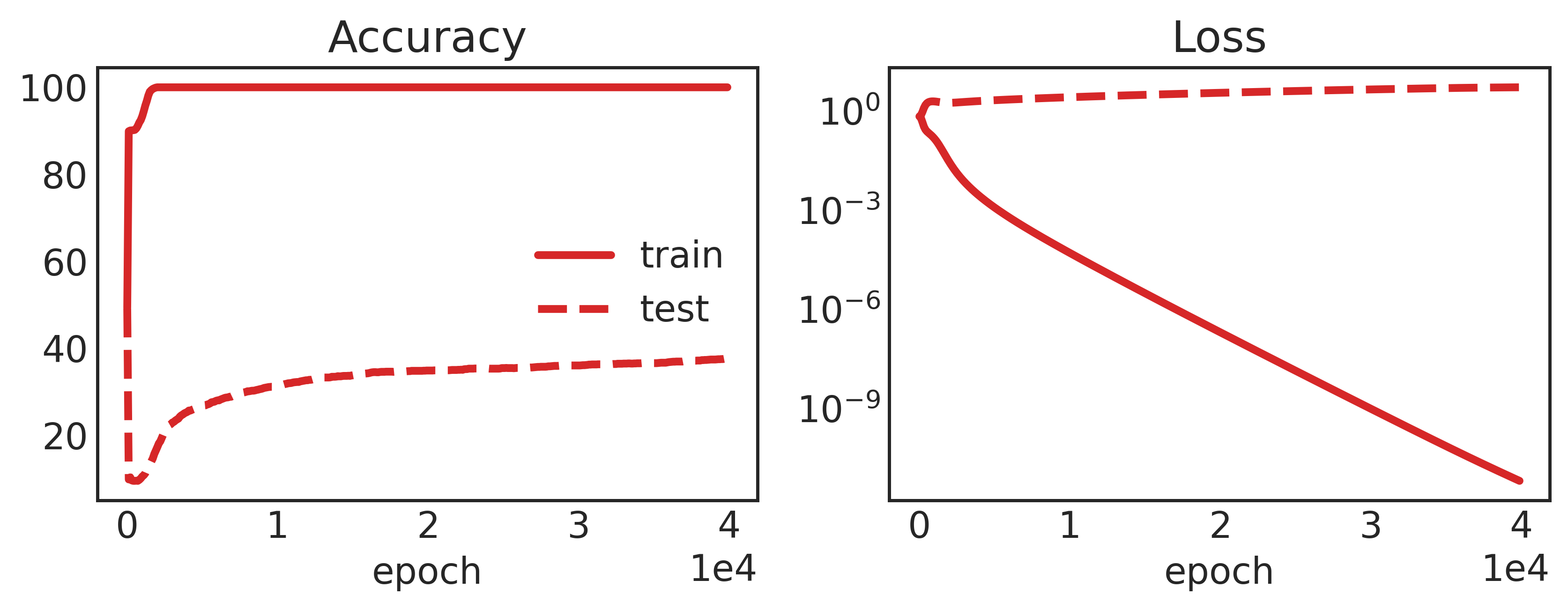}
         \caption{Average accuracy and loss curves.}
         \label{fig:nonlin-sim-example}
     \end{subfigure}
\begin{subfigure}[b]{0.6\textwidth}
\centering
\includegraphics[width=1.06\textwidth]{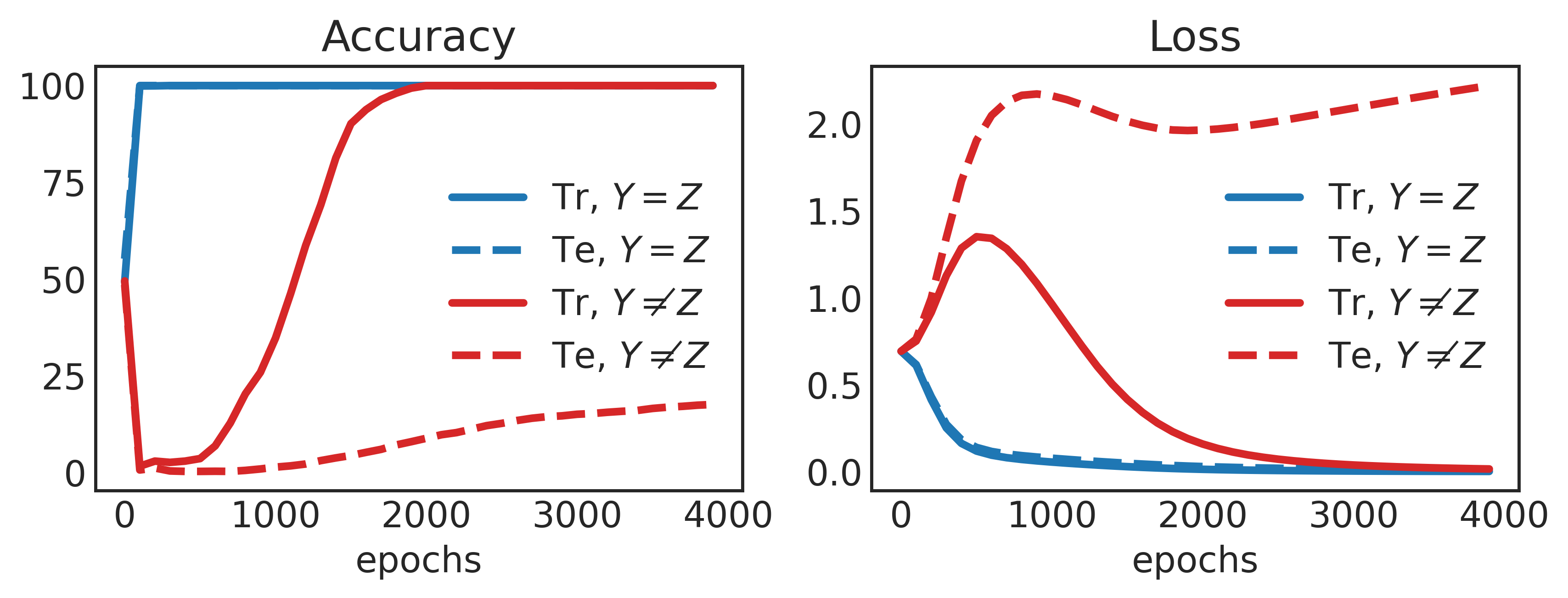}
     \caption{Accuracy and loss on shortcut and leftover groups.}
\label{fig:nonlin-sim-example-pergroup}
 \end{subfigure}
\caption{
\small
Training a two-layer neural network with \derm{} on data from \cref{eq:sim-example}.
The model achieves $100\%$ train accuracy but $<40\%$ test accuracy even after $40,000$ epochs.
The plot below zooms in on the first $4000$ epochs and shows that the model drives down loss on the test shortcut  groups but not on the test leftover group. This shows that the model uses the shortcut to classify the shortcut group and noise for the leftover.
}
	\label{fig:nonlin-erm-failure}
 \vspace{-10pt}
\end{figure}

In \cref{fig:sigstitch}, we compare \derm{} to \stitch{}.
In \cref{fig:sd} and \cref{fig:sd-log}, compare \gls{sd} and \textsc{marg-log} respectively to \derm{}.
The left panel of all figures shows that \gls{cc} achieves better test accuracy than \derm{}, while the right most panel shows that the test loss is better on the leftover group using \gls{cc}. 
Finally, the middle panel shows the effect of controlling margins in training; namely, the margins on the training data do not go to $\infty$, evidenced by the training loss being bounded away from $0$.

\begin{figure}[ht]
    \centering
    \includegraphics[width=0.9\textwidth]{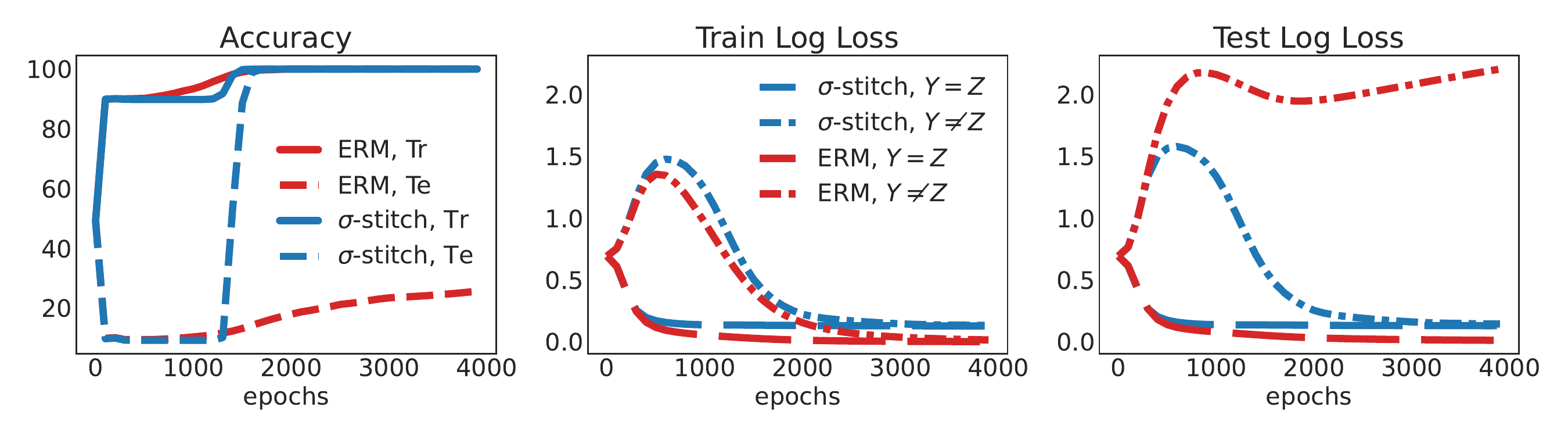}
    \caption{
    A neural network trained with \stitch{} depend on the perfect stable feature to achieve perfect test accuracy, unlike \derm{}.
The middle panel shows that \stitch{} does not let the loss on the training shortcut group to go to zero, unlike \derm{}, and the right panel shows the test leftover group loss is better.
    }
    \label{fig:sigstitch}
\end{figure}

\begin{figure}[ht]
    \centering
        \includegraphics[width=0.9\textwidth]{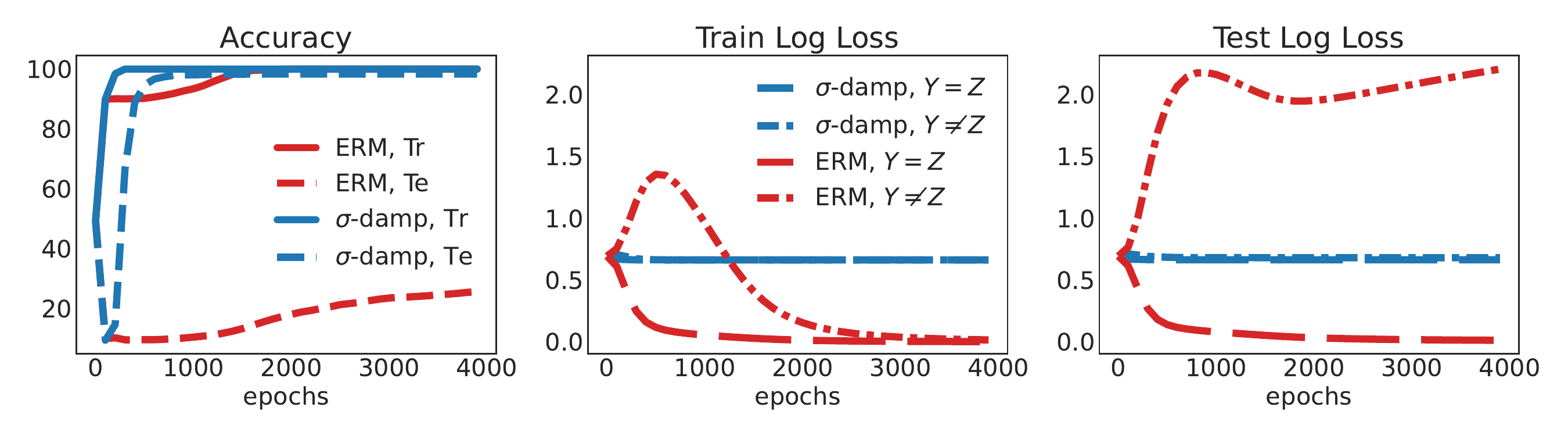}
    \caption{
A neural network trained with \sigdamp{} depend on the perfect stable feature to achieve perfect test accuracy whereas \derm{} performs worse than random chance.
The middle panel shows that \sigdamp{} does not let the loss on the training shortcut group to go to zero, unlike vanilla \derm{}, and the right panel shows the test-loss is better for the leftover group.
    }
    \label{fig:sigdamp}
\end{figure}

\begin{figure}[ht]
    \centering
        \includegraphics[width=0.9\textwidth]{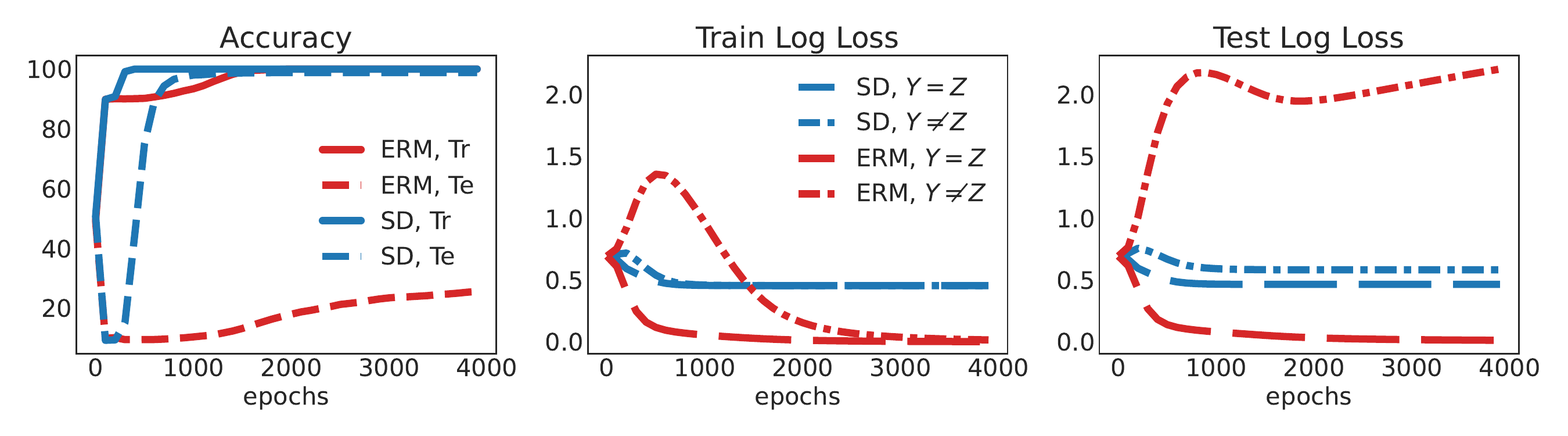}
    \caption{
A neural network trained with \gls{sd} depend on the perfect stable feature to achieve perfect test accuracy whereas \derm{} performs worse than random chance.
The middle panel shows that \gls{sd} does not let the loss on the training shortcut group to go to zero, unlike vanilla \derm{}, and the right panel shows the test-loss is better for the leftover group.
    }
    \label{fig:sd}
    \vspace{-10pt}
\end{figure}

\begin{figure}[ht]
\centering
    \includegraphics[width=0.9\textwidth]{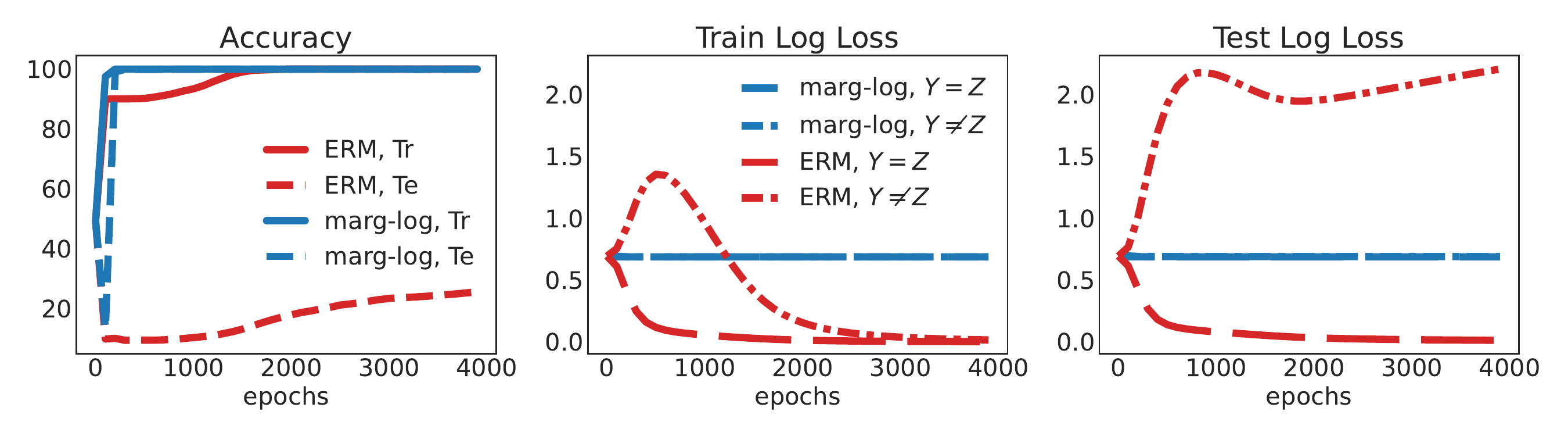}
    \caption{
A neural network trained with \textsc{marg-log} depend on the perfect stable feature to achieve perfect test accuracy whereas \derm{} performs worse than random chance.
The middle panel shows that \textsc{marg-log} does not let the loss on the training shortcut group to go to zero, unlike \derm{}, and the right panel shows the test-loss is better for the leftover group.
    }
    \label{fig:sd-log}
\end{figure}

\newacronym{gs}{GS}{gradient starvation}

\subsection{Spectral decoupling for a linear model on the linear \gls{dgp} in \cref{eq:sim-example}.}\label{appsec:spectral-decoupling}

We first show that a linear classifier trained with \gls{sd} achieves $100\%$ test accuracy while \derm{} performs worse than chance on the test data; so, \gls{sd} builds models with more dependence on the stable perfect feature, compared to \gls{erm}.
Next, we outline the assumptions for the \gls{gs} regime from \citet{pezeshki2021gradient} and then instantiate it for a linear model under the data generating process in 
\cref{eq:sim-example}, showing that the assumptions for the GS-regime are violated.

\Cref{fig:lin-sd} shows the results of training a linear model with \gls{sd} on training data of size $1000$ sampled as per \cref{eq:sim-example} from $p_{\rho=0.9}$ with $d=300$; the test data also has a $1000$ samples but comes from $p_{\rho=0.1}$.
\Cref{fig:lin-sd} shows that \gls{sd} builds models with improved dependence on the perfect stable feature, as compared to \gls{erm}, to achieve $100\%$ test accuracy.

\newacronym{ntk}{ntk}{neural-tangent-kernel}
\newacronym{ntrf}{ntrf}{neural-tangent-random-feature}
\newacronym{svd}{svd}{singular value decomposition}

\subsubsection{The linear example in \Cref{eq:sim-example} violates the gradient starvation regime.}

\paragraph{Background on \cite{pezeshki2021gradient}.}
With the aim of explaining why \gls{erm}-trained neural networks depend more on one feature over a more informative one, 
\citet{pezeshki2021gradient} derive solutions to $\ell_2$-regularized logistic regression in the \gls{ntk}; they let the regularization coefficient be small enough for the regularized solution to be similar in direction to the unregularized solution.
Given $n$ samples $\mby^i, \mbx^i$, let $\mbY$ be a diagonal matrix with the labels on its diagonal, $\mbX$ be a matrix with $\mbx^i$ as its rows, and $\hat{\mby}(\mbX, \theta) = \ftheta(\mbX)$ be the $n$-dimensional vector of function outputs where each element is $\hat{\mby}^i = \ftheta(\mbx^i)$.
In gradient-based training in the \gls{ntk} regime, the vector of function outputs of the network with parameters $\theta$ can be approximated as $\hat{\mby} = \Phi_0 \theta$, where $\Phi_0$ is the \gls{ntrf} matrix at initialization: \[\Phi_0 = \frac{\partial \hat{\mby}(\mbX, \theta_0)}{\partial \theta_0}\]

To define the features, the strength (margin) of each feature, and how features appear in each sample, \cite{pezeshki2021gradient} compute the \gls{svd} of the \gls{ntrf} $\Phi_0$ multiplied by the diagonal-label matrix $\mbY$:
\begin{align}
	\mbY \Phi_0 = \mbU \mbS \mbV^\top.
\end{align}
The rows of $\mbV$ are features, the diagonal elements of $\mbS$ are the strengths of each feature and the $i$th row of $\mbU$ denotes how each feature appears in the \gls{ntrf} representation of the $i$th sample.
\glsreset{gs}

To study issues with the solution to $\ell_2$-regularized logistic regression, \citet{pezeshki2021gradient} define the \gls{gs} regime.
Under the \gls{gs} regime, they assume $\mbU$ is a perturbed identity matrix that is also unitary: for a small constant $\delta <<1$, such a matrix has all diagonal elements $\sqrt{1-\delta^2}$ and the rest of the elements are of the order $\delta$ such that the rows have unit $\ell_2$-norm.

\allowdisplaybreaks

\paragraph{The \gls{gs} regime is violated in \cref{eq:sim-example}.}

When $\ftheta$ is linear, $\ftheta(\mbx) = \theta^\top \mbx$, the \gls{ntrf} matrix is
\[\frac{\partial \hat{\mby}(\mbX, \theta_0)}{\partial \theta_0} = \frac{\partial\mbX \theta_0}{\partial \theta_0} = \mbX.\]
In this case, let us look at an implication of $\mbU$ being a perturbed identity matrix that is also unitary, as \citet{pezeshki2021gradient} assume.
With $(\mbu^i)^\top$ as the $i$th row of $\mbU$, the transpose of $i$th sample can be written as $(\mbx^i)^\top=(\mbu^i)^\top\mbS\mbV$.
\cite{pezeshki2021gradient} assume that $\delta<<1$ in that the off-diagonal terms of $\mbU$ are small  perturbations such that off-diagonal terms of $\mbU (\mbS^2 + \lambda\mbI ) \mbU^\top$  have magnitude much smaller than $1$, meaning that the terms $|(\mbu^i)^\top \mbS^2 (\mbu^j) + \lambda| << 1$ for $i\not=j$ and positive and small $\lambda<<1$.

Then, 
\begin{align}
	|\mby^i\mby^j(\mbx^i)^\top\mbx^j| & = 	|(\mbx^i)^\top\mbx^j| 
\\
		& = |(\mbu^i)^\top\mbS\mbV^\top \mbV \mbS \mbu^j|
\\
	 & = |(\mbu^i)^\top\mbS^2\mbu^j|
\\
	& << 1
\end{align}
In words, this means that any two samples $\mbx^i, \mbx^j$ are nearly orthogonal.
Now, for samples from \cref{eq:sim-example}, for any $i,j$ such that $\mbz^j = \mbz^i$ and $\mby^i= \mby^j$,
\begin{align}
	\left|(\mbx^i)^\top\mbx^j\right| & = \left|B^2 \mbz^i\mbz^j + \mby^i\mby^j +  (\mbdelta^i)^\top\mbdelta^j\right| 
 \geq  |100 + 1 + (\mbdelta^i)^\top\mbdelta^j|
\end{align}
As $\mbdelta$ are isotropic Gaussian vectors, around half the pairs $i,j$ will have $(\mbdelta^i)^\top\mbdelta^j > 0$ meaning $\left|(\mbx^i)^\top\mbx^j\right| > 101$.
This lower bound implies that $\mbU$ is not a perturbed identity matrix for samples from \cref{eq:sim-example}.
This violates the setup of the gradient starvation regime from \citep{pezeshki2021gradient}.

Thus, the linear \gls{dgp} in \cref{eq:sim-example} does not satisfy the conditions for the \gls{gs} regime that is proposed in \citep{pezeshki2021gradient}.
The \gls{gs} regime blames the coupled learning dynamics for the different features as the cause for \derm{}-trained models depending more on the less informative feature.
\citet{pezeshki2021gradient} derive \glsreset{sd}\gls{sd} to avoid coupling the training dynamics, which in turn can improve a model's dependence on the perfect feature.
\gls{sd} adds a penalty to the function outputs which \cite{pezeshki2021gradient} show decouples training dynamics for the different features as defined by the \gls{ntrf} matrix:
\[\ell_{\text{\gls{sd}}}(\mby, \ftheta(\mbx)) = \log(1 + \exp(\mby\ftheta)) + \lambda|\ftheta(\mbx)|^2\]
As \cref{eq:sim-example} lies outside the \gls{gs} regime, the success of \gls{sd} on data from \cref{eq:sim-example} cannot be explained as a consequence of avoiding the coupled training dynamics in the \gls{gs} regime \citet{pezeshki2021gradient}.
However, looking at \gls{sd} as \gls{cc}, the success of \gls{sd}, as in \cref{fig:lin-sd}, is explained as a consequence encouraging uniform margins.

\begin{figure}[t]
\centering
	\includegraphics[width=0.85\textwidth]{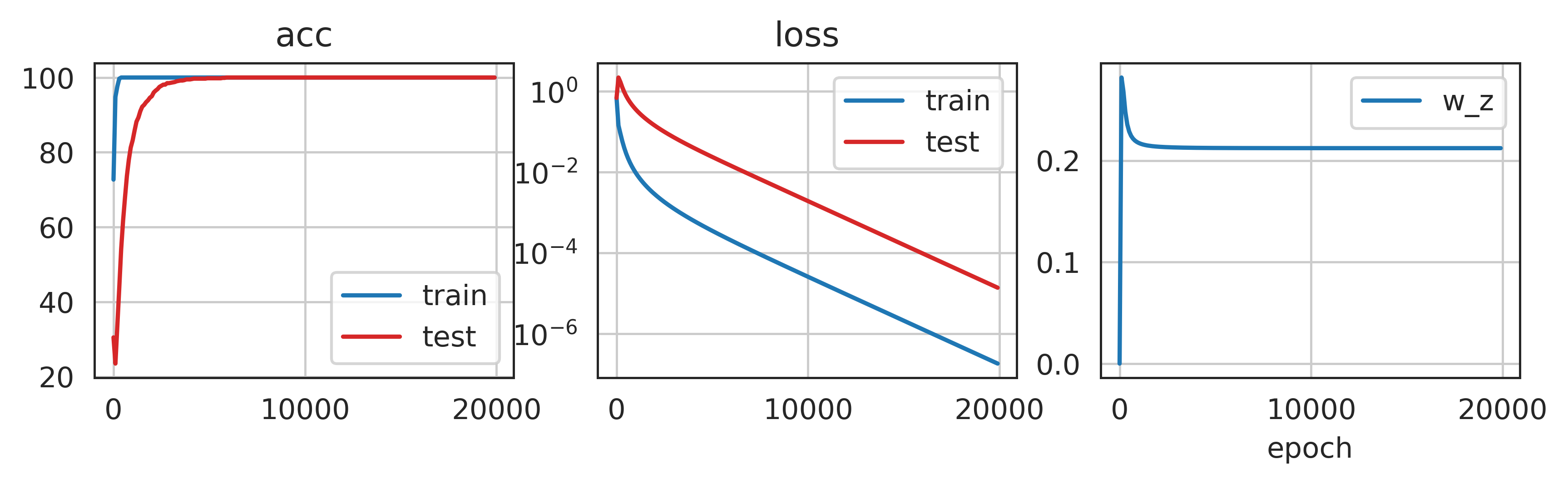}
\caption{
\small
With $d=200$ and $n=1000$, a linear classifier can still depend on the shortcut feature and achieve $100\%$ test accuracy.
\citet{nagarajan2020understanding} consider linearly separable data and formalize geometric properties of the data that make max-margin classifiers give non-zero weight to the shortcut feature $(\mbw_z > 0)$.
In their example, it is unclear when $\mbw_z > 0$ leads to poor accuracy in the leftover group because \citet{nagarajan2020understanding} do not separate the model’s dependence on the stable feature from the dependence on noise.
The example here gives an example where $\mbw_z > 0$ but test accuracy is $100\%$.
demonstrating that guarantees on test leftover group error require comparing $\mbw_y$ and $\mbw_z$; the condition $\mbw_z>0$ alone is insufficient. 
}
\label{fig:wz-pos-case}
\end{figure}

\paragraph{An example of perfect test accuracy even with dependence on the shortcut.}
In \cref{fig:wz-pos-case}, we train a linear model with \derm{} on data from \cref{eq:sim-example}, showing that even when shortcut dependence is non-zero, test leftover group accuracy can be $100\%$.
\citet{nagarajan2020understanding} consider linearly separable data and formalize geometric properties of the data that make max-margin classifiers give non-zero weight to the shortcut feature $(\mbw_z > 0)$.
In their example, it is unclear when $\mbw_z > 0$ leads to poor accuracy in the leftover group because \citet{nagarajan2020understanding} do not separate the model’s dependence on the stable feature from the dependence on noise.
The example in \cref{fig:wz-pos-case} gives an example where $\mbw_z > 0$ but test accuracy is $100\%$, demonstrating that guarantees on test leftover group error require comparing $\mbw_y$ and $\mbw_z$; the condition $\mbw_z>0$ alone is insufficient. 
In contrast, theorem 1 characterizes cases where leftover group accuracy is worse than random even without overparameterization.

\subsection{Experimental details}\label{appsec:exps}

\subsubsection{Background on \glsreset{jtt}\gls{jtt} and \glsreset{cnc}\gls{cnc}}\label{appsec:exps-twostage}

\paragraph{\gls{jtt}}
\citet{liu2021just} develop \gls{jtt} with the aim of building models robust to subgroup shift, where the mass of disjoint subgroups of the data changes between training and test times.
To work without training group annotations, \gls{jtt} assumes \gls{erm} builds models with high worst-group error.
With this assumption, \gls{jtt} first builds an "identification" model via \gls{erm} to pick out samples that are misclassified due to model's dependence on the shortcut.
Then, \gls{jtt} trains a second model again via \gls{erm} on the same training data with the loss for the misclassified samples upweighted (by constant $\lambda$).
As \citet{liu2021just} point out, the number of epochs to train the identification model and the upweighting constant are hyperparameters that require tuning using group annotations.
As \citet{liu2021just,pmlr-v162-zhang22z} show that \gls{jtt} and \gls{cnc} outperforms \gls{lff} and other two-stage \sms{} (\citep{pmlr-v162-zhang22z}), so we do not compare against them.

\paragraph{\glsreset{cnc}\gls{cnc}}
In a fashion similar to \gls{jtt}, the first stage of \gls{cnc} is to train a model with regularized \gls{erm} to predict based on spurious attributes, i.e. shortcut features.
\citet{pmlr-v162-zhang22z} develop a contrastive loss to force the model to have similar representations across samples that share a label but come from different groups (approximately inferred by the first-stage \gls{erm} model).
Formally, the first-stage model is used to approximate the spurious attributes in one of two ways: 1) predict the label with the model, 2) cluster the representations into as many clusters as there are classes, and then use the cluster identity.
The latter technique was first proposed in \citep{sohoni2020no}.
For an anchor sample $(\mby^i, \mbx^i)$ of label $\mby=y$, positive samples $P_i$ are those than have the same label but have the predicted spurious attribute is a different value: $\hat{z}\not=y$.
Negatives $N_i$ are those that have a different label but the spurious attribution is the same: $\hat{z} = y$.
For a temperature parameter $\tau$ and representation function $r_\theta$, the per-sample contrastive loss for \gls{cnc} is:
\[\ell_{cont}(r_\theta, i) = \E_{\mbx^p \sim P_i} \left[- \log  \frac{\exp(\nicefrac{r_\theta(\mbx^i)^\top r_\theta(\mbx^p)}{\tau} )}{\sum_{n \in N_i} \exp\left(\nicefrac{r_\theta(\mbx^i)^\top r_\theta(\mbx^n)}{\tau}\right) + \sum_{p\in P_i} \exp\left(\nicefrac{r_\theta(\mbx^i)^\top r_\theta(\mbx^p) }{\tau}\right)}\right].\]

The samples $i$ are called \textit{anchors}.
For a scalar $\lambda$ to trade off between contrastive and predictive loss, the overall per-sample loss in the second-stage in \gls{cnc} is 
\[\lambda \ell_{cont}(r_\theta, i)  + (1-\lambda) \ell_{log-loss}(\mby^iw^\top r_\theta(\mbx^i)).\]

\paragraph{\Gls{cnc} uses hyperparameters informed by dataset-specific empirical results from prior work.}
The original implementation of \gls{cnc} from \citet{pmlr-v162-zhang22z} uses specific values of first-stage hyperparameters like weight decay and early stopping epoch for each dataset by using empirical results from prior work \citep{sagawa2019distributionally,liu2021just}.
The prior work finds weight-decay and early stopping epoch which lead \derm{} models to achieve low test worst-group accuracy, implying that the model depends on the spurious attribute.
This means the first-stage models built in \gls{cnc} are pre-selected to pay attention to the spurious attributes.
For example, \citep{pmlr-v162-zhang22z} point out that the first-stage model they use for Waterbirds predicts the spurious feature with an accuracy of $94.7\%$.

Without using dataset-specific empirical results from prior work, choosing \gls{lr} and \gls{wd} requires validating through the whole \gls{cnc} procedure.
We let \gls{cnc} use the same \gls{lr} and \gls{wd} for both stages and then validate the choice using  validation performance of the second-stage model. 
This choice of hyperparameter validation leads to a similar number of validation queries for all methods that mitigate shortcuts.

\subsubsection{Training details}\label{appsec:training-details}

\paragraph{Variants of \gls{cc} to handle label imbalance.}

The three datasets that we use in our experiments --- Waterbirds, CelebA, and Civilcomments --- all have an imbalanced (non-uniform) marginal distribution over the label; for each dataset, 
\[\max_{\text{class}\in \{-1,1\}} p(\mby=\text{class})>0.75.\]
When there is sufficiently large imbalance, restricting the margins on all samples could bias the training to reduce loss on samples in the most-frequent class first and overfit on the rest of the samples.
This could force a model to predict the most frequent class for all samples, resulting in high worst-group error.

To prevent such a failure mode, we follow \citep{pezeshki2021gradient} and define variants of \sigdamp{}, \textsc{marg-log}, and \stitch{} that have either 1) different maximum margins for different classes or 2) different per-class loss values for the same margin value.
Mechanically, these variants encourage uniform margins within each class, thus encouraging the model to rely less on the shortcut feature. We give the variants here for labels taking values in $\{-1,1\}$:

\begin{enumerate}
	\item With per-class temperatures $T_{-1},T_1>0$ the variant of \sigdamp{} is 
\begin{align*}
\text{with } & \ftheta = {w_f}^\top {r_\theta(\mbx)},
\\
& \ell_{\text{\sigdamp{}}}(\mby,\ftheta) =
\ell_{log}
    \left[
       T_{\mby}*1.278\mby 
        \ftheta
        \left( 
        1 -\sigma\left(1.278*{\mby\ftheta}\right)\right)
    \right]
\end{align*}
The $1.278$ comes in to make sure the maximum input to log-loss occurs at $\ftheta=1$.
However, due to the different temperatures $T_1\not=T_{-1}$, achieving the same margin on all samples produces lower loss on the class with the larger temperature.

\item
    With per-class temperatures $T_{-1},T_1>0$ the variant of \stitch{} is 
        \begin{align*}
            \text{with } & 
                \ftheta = {w_f}^\top {r_\theta(\mbx)}, \qquad
            \\
            & 
                \ell_{\text{\stitch}}(\mby\ftheta) =
             \ell_{log}\left(T_{\mby}\left[\quad\mathbf{1}[\mby\ftheta(\mbx) < 1]\times \mby\ftheta(\mbx) \right. + \left.\mathbf{1}[\mby\ftheta(\mbx) > 1]\times (2-\mby\ftheta(\mbx))\,\,\right]\right)
        \end{align*}
\item
With per-class function output targets $\gamma_{-1},\gamma_1>0$ the variant of \textsc{marg-log} is
\begin{align*}
\text{with } & \ftheta = {w_f}^\top {r_\theta(\mbx)}, \qquad 
\\
&\ell_{\textsc{marg-log}}(\mby\ftheta) =
 \ell_{log}(\mby\ftheta) + \lambda \log(1 + |\ftheta - \gamma_{\mby}|^2).
\end{align*}

\end{enumerate}

These per-class variants are only for training; at test time, the predicted label is $\texttt{sign}(\ftheta)$.

\paragraph{Details of the vision and language experiments.}

We use the same datasets from \cite{liu2021just}, downloaded via the scripts in the code from \citep{idrissi2021simple};
see \citep{idrissi2021simple} for sample sizes and the group proportions.
For the vision datasets, we finetune a resnet50 from Imagenet-pretrained weights and for Civilcomments, we finetune a BERT model.\\

\glsreset{lr}\glsreset{wd}
\paragraph{Optimization details.}
For all methods and datasets, we tune over the following \gls{wd} parameters: $10^{-1},10^{-2},10^{-3},10^{-4}$
For the vision datasets, we tune \gls{lr} over $10^{-4},10^{-5}$ and for CivilComments, we tune over $10^{-5},10^{-6}$.
For CivilComments, we use the AdamW optimizer while for the vision datasets, we use the Adam optimizer; these are the standard optimizers for their respective tasks \citep{puli2022outofdistribution,gulrajani2020search}.
We use a batch size of $128$ for both CelebA and Waterbirds, and train for $20$ and $100$ epochs respectively. 
For CivilComments we train for $10$ epochs with a batch size of $16$.\\

\begin{figure}[t]
\vspace{-20pt}
\centering
	\includegraphics[width=0.9\textwidth]{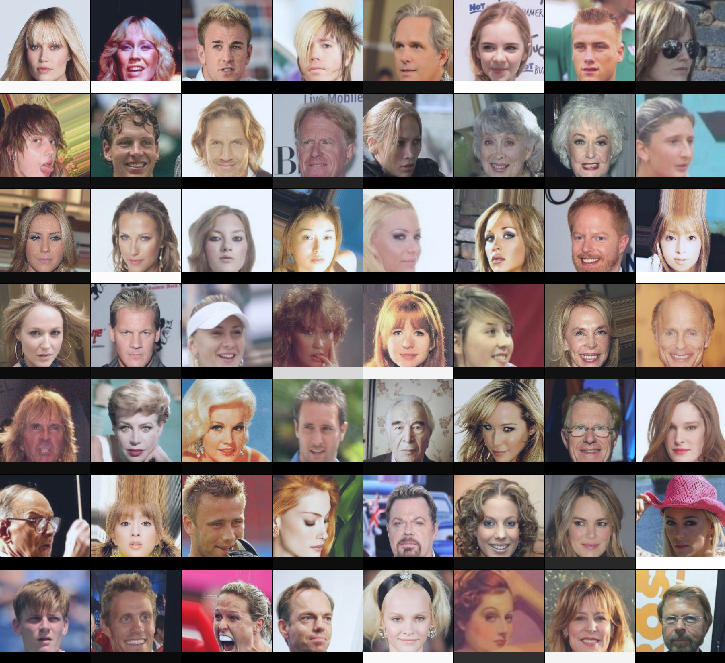}
	\caption{Images mis-classified by a model trained on CelebA data with equal group sizes, i.e. without a shortcut. Samples with blonde as the true label have a white strip at the bottom while samples with non-blonde as the true label have a black strip at the bottom. The figure demonstrates that many images with blonde people in the image have the non-blonde label, thus demonstrating label noise. 
 For example, see a blonde man in the first row that is labelled non-blonde and a non-blonde lady in the third row that is lablled blonde.
 Yet, \gls{cc} improves over \gls{erm} for many \gls{lr} and \gls{wd} combinations; see \cref{eq:sensitivity}.}
	\label{fig:label-noise}
\end{figure}

\paragraph{Per-method Hyperparameters.}

Like in \citep{pezeshki2021gradient}, the per-class temperatures $T_{-1}, T_1$ for \sigdamp{} and \stitch{}, 
and the function output targets $\gamma_{-1}, \gamma_{1}$ for \textsc{marg-log} are hyperparameters that we tune using the worst-group accuracy or label-balanced average accuracy computed on the validation dataset, averaged over $2$ seeds.
\begin{enumerate}
	\item For \stitch{}, we select from $T_{-1} \in \{1, 2\}$ and $T_1 \in \{2, 4, 8, 12\}$ such that $T_1 > T_{-1}$.
 
\item  For \sigdamp{}, we search over $T_{-1} \in \{1, 2\}$ and $T_1 \in \{2, 4\}$ such that $T_1 > T_{-1}$.

\item  For \gls{sd} and \textsc{marg-log}, we search over $\gamma_{-1} \in \{-1, 0, 1\}$ and $\gamma_1 \in \{1, 2, 2.5, 3\}$ for the image datasets and $\gamma_1 \in \{1, 2\}$ for the text dataset, and the penalty coefficient is set to be $\lambda=0.1$

\item  For \gls{jtt}, we search over the following parameters: the number of epochs $T\in \{1,2\}$ for CelebA and Civilcomments and $T\in\{10,20,30\}$ for Waterbirds, and the upweighting constant $\lambda\in\{20,50,100\}$ for the vision datasets and $\lambda\in\{4,5,6\}$ for Civilcomments.

\item For \gls{cnc}, we search over the same hyperparameter as \citep{pmlr-v162-zhang22z}  : the temperature in $\tau \in \{0.05, 0.1\}$, the contrastive weight $\lambda \in \{0.5, 0.75\}$, and the gradient accumulation steps $s\in\{32,64\}$.
For the language task in Civilcomments, we also try one additional $s=128$.

\end{enumerate}

\subsection{\Gls{cc} improves over \derm{} on CelebA even without the stable feature being perfect.}
CelebA is a perception task in that the stable feature is the color of the hair in the image.
But unlike the synthetic experiments, \gls{cc} does not achieve a $100\%$ test accuracy on CelebA.
We investigated this and found that CelebA in fact has some label noise.

We trained a model via the \gls{cc} method \sigdamp{} on CelebA data with no shortcut; this data is constructed by subsampling the groups to all equal size, ($5000$ samples).
This achieves a test worst-group accuracy of $89\%$.
We visualized the images that were misclassified by this model and found that many images with blond-haired people were classified as having non-blonde hair.
\Cref{fig:label-noise} shows $56$ misclassified images where samples with blonde as the true label have a white strip at the bottom while samples with non-blonde as the true label have a black strip at the bottom.
The figure shows that images with blonde people can have the non-blonde label, thus demonstrating label noise.
Thus, \gls{cc} improves over \gls{erm} even on datasets like CelebA where the stable features do not determine the label.

\begin{figure}[ht]
\centering
	\includegraphics[width=0.6\textwidth]{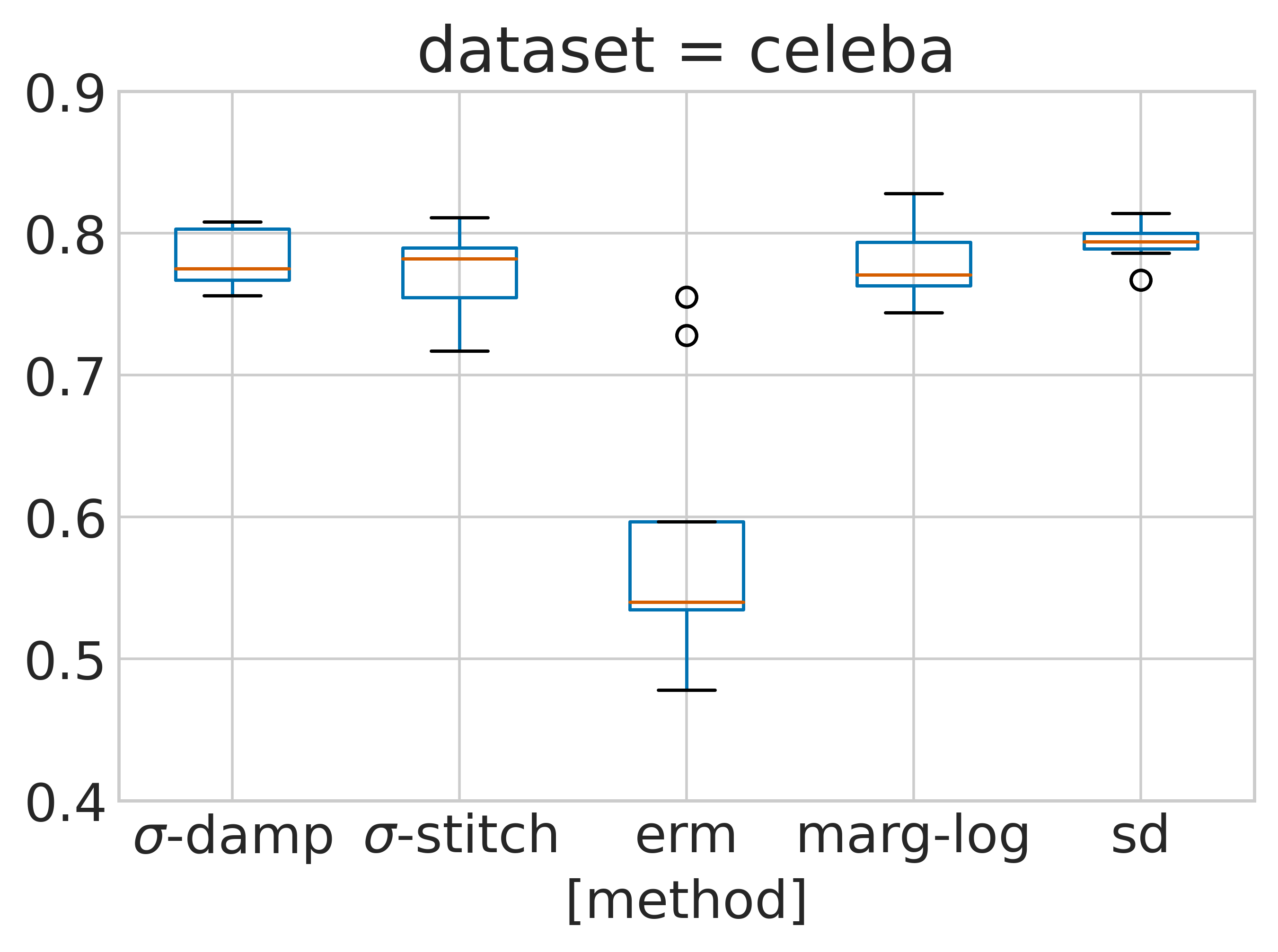}
\caption{Test worst-group accuracy on CelebA of \derm{} and \gls{cc} for different values of \gls{lr} and \gls{wd}. \Derm{}'s performance changes more with \gls{lr} and \gls{wd} than \gls{cc}, which shows that \derm{} is more sensitive than \gls{cc}.
Only 2 combinations of \gls{lr} and \gls{wd} improve  \gls{erm} beyond a test worst-group accuracy of $60\%$, while every \gls{cc} method achieves more than $70\%$ test worst-group accuracy for every combination of \gls{lr} and \gls{wd}.
}
\label{eq:sensitivity}
\end{figure}

\subsection{Sensitivity of \gls{erm} and \gls{cc} to varying \gls{lr} and \gls{wd}}\label{appsec:additional}
In \cref{eq:sensitivity}, we compare the test worst-group accuracy of \derm{} and \gls{cc} on CelebA, for different values of \gls{lr} and \gls{wd}.
There are $8$ combinations of \gls{lr} and \gls{wd} for which \gls{erm} is run.
For each combination of \gls{lr} and \gls{wd}, the hyperparameters of the \gls{cc} method (values of $\lambda, T, v$) are tuned using validation group annotations, and the test worst-group accuracy corresponds to the best method hyperparameters.
\Derm{}'s performance changes more with \gls{lr} and \gls{wd} than \gls{cc}, which shows that \derm{} is more sensitive than \gls{cc}.
Only 2 combinations of \gls{lr} and \gls{wd} improve  \gls{erm} beyond a test worst-group accuracy of $60\%$, while every \gls{cc} method achieves more than $70\%$ test worst-group accuracy for every combination of \gls{lr} and \gls{wd}.

\end{document}